\documentclass[sigconf]{acmart}
\usepackage[linesnumbered,ruled]{algorithm2e}
\usepackage{amsmath}
\usepackage{amsfonts}
\usepackage{amsthm}

\usepackage{subcaption}
\usepackage{cleveref}
\usepackage{multirow}
\usepackage{wasysym}   %for三角形符号

\usepackage{booktabs}   %% For formal tables:
\usepackage{subcaption} %% For complex figures with subfigures/subcaptions
\begin{document}
	
	%% Title information
	\title[Short Title]{$\varepsilon$-weakened Robustness of Deep Neural Networks}         %% [Short Title] is optional;
	%% when present, will be used in
	%% header instead of Full Title.
	%\titlenote{with title note}             %% \titlenote is optional;
	%% can be repeated if necessary;
	%% contents suppressed with 'anonymous'
	%\subtitle{Subtitle}                     %% \subtitle is optional
	%\subtitlenote{with subtitle note}       %% \subtitlenote is optional;
	%% can be repeated if necessary;
	%% contents suppressed with 'anonymous'

	%% Author information
	%% Contents and number of authors suppressed with 'anonymous'.
	%% Each author should be introduced by \author, followed by
	%% \authornote (optional), \orcid (optional), \affiliation, and
	%% \email.
	%% An author may have multiple affiliations and/or emails; repeat the
	%% appropriate command.
	%% Many elements are not rendered, but should be provided for metadata
	%% extraction tools.
	
	%% Author with single affiliation.
	\author{Pei Huang}
	%\authornote{with author1 note}          %% \authornote is optional;
	%% can be repeated if necessary
	%\orcid{nnnn-nnnn-nnnn-nnnn}             %% \orcid is optional
	\affiliation{
		%\position{Position1}
		\department{State Key Laboratory of Computer Science}              %% \department is recommended
		\institution{Institution of Software, CAS}            %% \institution is required
		%\streetaddress{Street1 Address1}
		\city{Beijing}
		%\state{State1}
		%\postcode{Post-Code1}
		\country{China}                    %% \country is recommended
	}
	\email{haungpei@ios.ac.cn}          %% \email is recommended
	
		\author{Yuting Yang}
	%\authornote{with author1 note}          %% \authornote is optional;
	%% can be repeated if necessary
	%\orcid{nnnn-nnnn-nnnn-nnnn}             %% \orcid is optional
	\affiliation{
		%\position{Position1}
		\department{Key Laboratory of Intelligent Information}              %% \department is recommended
		\institution{ Institute of Computing Technology, CAS}            %% \institution is required
		%\streetaddress{Street1 Address1}
		\city{Beijing}
		%\state{State1}
		%\postcode{Post-Code1}
		\country{China}                    %% \country is recommended
	}
	\email{yangyuting@ict.ac.cn}          %% \email is recommended
	
	%% Author with two affiliations and emails.
	\author{Minghao Liu}
%\authornote{with author1 note}          %% \authornote is optional;
%% can be repeated if necessary
%\orcid{nnnn-nnnn-nnnn-nnnn}             %% \orcid is optional
\affiliation{
	%\position{Position1}
	\department{State Key Laboratory of Computer Science}              %% \department is recommended
	\institution{Institution of Software, CAS}            %% \institution is required
	%\streetaddress{Street1 Address1}
	\city{Beijing}
	%\state{State1}
	%\postcode{Post-Code1}
	\country{China}                    %% \country is recommended
}
\email{liumh@ios.ac.cn}          %% \email is recommended         %% \email is recommended
		\author{Fuqi Jia}
	%\authornote{with author1 note}          %% \authornote is optional;
	%% can be repeated if necessary
	%\orcid{nnnn-nnnn-nnnn-nnnn}             %% \orcid is optional
	\affiliation{
		%\position{Position1}
		\department{State Key Laboratory of Computer Science}              %% \department is recommended
		\institution{Institution of Software, CAS}            %% \institution is required
		%\streetaddress{Street1 Address1}
		\city{Beijing}
		%\state{State1}
		%\postcode{Post-Code1}
		\country{China}                    %% \country is recommended
	}
	\email{jiafq@ios.ac.cn}  
	
			\author{Feifei Ma}
	\authornote{Corresponding author}          %% \authornote is optional;
	%% can be repeated if necessary
	%\orcid{nnnn-nnnn-nnnn-nnnn}             %% \orcid is optional
	\affiliation{
		%\position{Position1}
		\department{State Key Laboratory of Computer Science}              %% \department is recommended
		\institution{Institution of Software, CAS}            %% \institution is required
		%\streetaddress{Street1 Address1}
		\city{Beijing}
		%\state{State1}
		%\postcode{Post-Code1}
		\country{China}                    %% \country is recommended
	}
	\email{maff@ios.ac.cn} 
	
			\author{Jian Zhang}
	\authornote{Corresponding author}          %% \authornote is optional;
	%% can be repeated if necessary
	%\orcid{nnnn-nnnn-nnnn-nnnn}             %% \orcid is optional
	\affiliation{
		%\position{Position1}
		\department{State Key Laboratory of Computer Science}              %% \department is recommended
		\institution{Institution of Software, CAS}            %% \institution is required
		%\streetaddress{Street1 Address1}
		\city{Beijing}
		%\state{State1}
		%\postcode{Post-Code1}
		\country{China}                    %% \country is recommended
	}
	\email{zj@ios.ac.cn} 
	
	%% Abstract
	%% Note: \begin{abstract}...\end{abstract} environment must come
	%% before \maketitle command
	\begin{abstract}
		This paper introduces a notation of $\varepsilon$-weakened robustness for analyzing the reliability and stability of deep neural networks (DNNs). Unlike the conventional robustness, which focuses on the ``perfect'' safe region in the absence of adversarial examples, $\varepsilon$-weakened robustness focuses on the region where the proportion of adversarial examples is bounded by user-specified $\varepsilon$. Smaller $\varepsilon$ means a smaller chance of failure. Under such robustness definition, we can give conclusive results for the regions where conventional robustness ignores. We prove that the $\varepsilon$-weakened robustness decision problem is PP-complete and give a statistical decision algorithm with user-controllable error bound. Furthermore, we derive an algorithm to find the maximum $\varepsilon$-weakened robustness radius. The time complexity of our algorithms is polynomial in the dimension and size of the network. So, they are scalable to large real-world networks. Besides, We also show its potential application in analyzing quality issues.
	\end{abstract}
	\keywords{deep neural network, robustness, reliability, testing, statistical method }  %% \keywords are mandatory in final camera-ready submission

	%% \maketitle
	%% Note: \maketitle command must come after title commands, author
	%% commands, abstract environment, Computing Classification System
	%% environment and commands, and keywords command.
	\maketitle

\section{Introduction}
Deep neural networks (DNNs) have been broadly applied in various domains. However, many studies have shown that these state-of-the-art models can be easily compromised by adding small perturbations \cite{Szegedy,GoodfellowSS14,ChenSZYH18}, i.e., visually imperceptible image perturbations can cause the misclassification of neural networks. Such unstable properties have raised serious concerns about the reliability of the application based on DNNs. Addressing for safe adoption of DNNs requires robustness metrics and efficient analyzing methods to understand the reliability and vulnerability of neural networks.

In recent years, a large body of prior work has focused on robustness which is defined on the absence of adversarial examples around a given point. These robust regions are absolutely safe against any malicious attacks. However, it is not enough to understand the reliability and stability of neural networks, we notice that: (1) The ``perfect" safe regions are usually very small for the real-world random perturbations. Take an example, for ImageNet task, more than 90\% regions of a well-trained ResNet are unsafe under the condition of the perturbation radius $2/255$ \cite{WongRK20}. In the real environment, a drop of rain on the input camera or the changes in light and shade can often shift the input out of these ``perfect" safe regions. In this situation, such robustness analyzing can not provide a conclusion about whether an NN can resist these perturbations well. (2) The definition of conventional robustness can suffer from the ``curse of dimensionality'', which means, in high-dimensional space, the majority part of the safe region around a given point can be ignored (Sec.\ref{Aanalyses}). (3) The robustness radius is strongly related to the location of a given point. So we cannot make a comprehensive judgment on the reliability of the NN in that region (Sec.\ref{Aanalyses}).

Unlike traditional programs, there is no basic theory to guarantee the correctness of deep neural networks. Even if a concept is PAC-learnable \cite{Valiant84}, the learning algorithm can only learn a close approximation to the target concept with generalization error. In 2019, David .et.al \cite{Stutz0S19} point out that on-manifold robustness is essentially generalization and adversarial examples are generalization errors. As the generalization error cannot be eliminated, it is difficult to eliminate all adversarial examples around a point at this stage. These adversarial examples distributed everywhere and make almost all the perfect safe regions in norm bounds too small, as if the network is just robust at the given test points. Many studies show that improving such robustness a little is usually accompanied by a sharp decline in accuracy \cite{WongRK20,ShafahiNG0DSDTG19} and make NN inapplicable in practice. So the pursuit of error-free robustness sometimes can be too severe for many applications.

%Many studies show that a small increase in the size of perfect safe regions (e.g. adversarial training) can be at the cost of a significant drop in accuracy. 

In this paper, we study the robustness from the perspective of quantification. We formalize a robustness metric called \textbf{$\varepsilon$-weakened robustness} ($\varepsilon$-robustness in short) which allows some adversarial examples around a test data point but the proportion is bounded by a user-specified $\varepsilon$. Smaller $\varepsilon$ indicates a smaller failure rate or the smaller chance that the neural network yields to non-crafted perturbations. For many practical applications, safety standards are defined on the specified low failure rates. For instance, in autonomous robotics and self-driving settings where there is some environment model and we would like to bound the probability of an adverse outcome \cite{kalra2016driving,koopman2019safety}. Thus, analyzing the $\varepsilon$-robustness of neural networks is very helpful for building a reliable system in deployment.

After introducing $\varepsilon$: (1) we can draw analysis results about the reliability of neural networks when the real-word perturbations often shift the inputs out of the perfect safe regions. (2) It is also helpful to alleviate the ``curse of dimensionality'' in conventional robustness analysis and reduce the dependence of analysis results on the locations of test points. (3) It not only can be applied to analyze whether the region around a correctly classified point is ``good'' enough to resist perturbations, but also to analyze whether the region around a misclassified point is so ``bad''. (4) It can be extended to the multi-labels version for analyzing the risk. Sometimes misclassification does not mean disaster. If the network recognizes a car as a truck, the self-driving car can still make a right decision to avoid it. The extended version $\varepsilon$-robustness can be applied to analyzing the risk of not recognizing a car as a vehicle by a neural network. 

We call determining whether a region is $\varepsilon$-robust as \textbf{$\varepsilon$-weakened robustness decision problem} and finding the maximum $\varepsilon$-robust radius as \textbf{$\varepsilon$-weakened robustness evaluation problem}. We prove that the $\varepsilon$-robustness decision problem is PP-complete, so it is not feasible to tackle these problems with precise methods for real-world large DNNs. Thus, we apply a statistical inference method with user-controllable error bounds (or confidence) $\alpha$ and $\beta$. $\alpha$ is the probability of \textbf{false positives} and $\beta$ is the probability of \textbf{false negative}. Such simulations and statistical inference can be promising in analyzing deep learning systems because each execution of a neural network takes a little time even when the model has millions of parameters. Besides, many potential executions can be simulated in parallel easily. 

%Compared with traditional software, this property is unique for deep learning systems because a simulation may take exponential time for a traditional program even when the code is only a few lines.

Our contributions can be summarized as follows:
\begin{itemize}
	\item We analyze the limitation of conventional robustness in reliability analysis and formalize $\varepsilon$-weakened robustness for DNNs. It can provide a more comprehensive analysis of the reliability and risk of DNN models.
	\item We prove that $\varepsilon$-weakened robustness decision problem is $PP$-complete.
	\item We propose a probabilistic algorithm for $\varepsilon$-weakened robustness decision problem with provable guarantees w.r.t $\alpha$ and $\beta$. The algorithm is polynomial in input dimension and size of the network. 
	\item We propose an algorithm to tackle $\varepsilon$-weakened robustness evaluation problem under the framework of querying decision oracle.
	\item We evaluate our methods on the popular neural networks and data sets. They are scalable to large DNNs with more than 100 layers. Furthermore, we conduct some experiments to show its potential applications in analyzing the quality issues of deep learning systems.
\end{itemize}

\section{Preliminaries}

\subsection{Neural Networks and Local Robustness}
Let $f: \mathbb{D}^n \rightarrow \mathbb{R}^m$ be a neural network classifier such that,
for a given input $x \in \mathbb{D}^n$, $f(x)=\{o_1(x),o_2(x),..., o_m(x)\}\in \mathbb{R}^m$ represents the confidence values for $m$ classification labels. The prediction of $x$ is given as $F(x)=\mathop{\arg\max}_{1\leq i \leq m} o_i(x)$. 
\begin{definition}[Local Robustness]
	Given a test point $x_*$ with predicted label $l_*$, the neural network is locally robust at point $x_*$ with respect to a distortion radius $r$ if the following formula holds:
	\begin{equation}\label{rob}
		\forall x_*' .\  x_*' \in B_p(x_*, r) \Rightarrow F(x_*')=l_*
	\end{equation}
	where $B_p(x_*, r)=\{x_*'\ | \left\| x_*'-x_*\right\|_p \leq r \}$ is an $\ell_p$-norm ball. 
\end{definition}
It means that any perturbed examples of $x_*$ with $\ell_p$-distortion $\Delta_p \leq r$ will not change the classification results. 

The \textbf{Local Robustness Evaluation (LRE)} problem over an input $x_*$ aims at finding the maximum radius of a safe $\ell_p$-norm ball with the center $x_*$. It can be formalized as an optimization problem:
\begin{equation}
	\begin{split}
		&\mathbf{max} \ \  r  \\
		&\mathbf{s.t}\ \ \forall x_*' .\  x_*' \in B_p(x_*, r) \Rightarrow F(x_*')=l_*
	\end{split}
\end{equation}
We call the maximized $r$ as \textbf{robustness radius}. Limited by computational complexity, finding the maximum radius is very difficult and most of the work is devoted to giving a lower bound.

\begin{definition}[Lipschitz continuity]
	Let $S\subset \mathbb{R}^n$ be a convex bounded closed set and let $f(x): S\rightarrow \mathbb{R}$ be a function on an open set containing $S$. Then $f$ is Lipschitz continuous if there exists a real constant $K>0$ such that  $\forall x,y \in S$:
	\begin{equation*}
		\forall x,y \in S, |f(y)-f(x)|\leq K \left\| y-x\right\|_p
	\end{equation*}
\end{definition}
\begin{definition}[Continuity]
	Let $S\subset \mathbb{R}^n$ be a convex bounded closed set and let $f(x): S\rightarrow \mathbb{R}$ be a function on an open set containing $S$. Then $f$ is continuous at point $x_0$ if $\forall \xi>0$, $\exists \eta>0$:
	\begin{equation*}
		\forall x \in S, \left\| x-x_0\right\|_p <\eta \Rightarrow  |f(x)-f(x_0)|<\xi
	\end{equation*}
\end{definition}
Lipschitz continuity implies continuity.

\begin{lemma}\label{lemma1}
	Let $f:\mathbb{R}^n \rightarrow \mathbb{R}^m$ be a neural network continuous at each point of the domain $S$ and $x_*$ is an input. The prediction of the network is given by $l=\mathop{\arg\max}_{1\leq i \leq m} o_i(x_*)$. If $\forall i \in \{1,2...m\}$, $o_l(x_*) > o_i(x_*)$ where $i\neq l$ holds, then there exists $\eta$ such that:
	\begin{equation*}
		\forall x \in S, \left\| x-x_*\right\|_p < \eta \Rightarrow  F(x)=l
	\end{equation*}
\end{lemma}
\begin{proof}
	Let $g_{i}(x)=o_l(x)-o_i(x)$ ($i\neq l $), we have $g_{i}(x)$ is continuous at point $x_*$ and $g_{i}(x_*)>0$. Assume $g_{i}(x_*)=\xi_i$, based on the definition of continuity there exist $\eta_i$ such that:
	\begin{equation*}
		\forall x \in S, \left\| x-x_*\right\|_p <\eta_i \Rightarrow  |g_{i}(x)-g_{i}(x_*)|<\xi_i
	\end{equation*}
	which implies that
	\begin{equation*}
		\forall x \in S, \left\| x-x_*\right\|_p <\eta_i \Rightarrow  g_{i}(x)>0
	\end{equation*}
	Let $\eta=\min_{i\neq l }{\{\eta_i\}}$, we have for all $i$ over ($i\neq l $):
	\begin{equation*}
		\forall x \in S, \left\| x-x_*\right\|_p <\eta \Rightarrow  o_l(x)>o_i(x)
	\end{equation*}
	So, there exists $\eta>0$, the classification result will not change when $\left\| x-x_*\right\|_p <\eta$
\end{proof}

Lemma \ref{lemma1} states several facts about continuous neural network and test point satisfying $o_l(x_*) > o_i(x_*)$:
(1) If a point is predicted correctly, there must be a neighborhood such that all points in it will be predicted correctly.
(2) If we find an adversarial example, there must be a neighborhood such that all points in it will be predicted incorrectly. 

Almost all neural networks defined on $ \mathbb{R}^n$ are continuous unless discontinuous activation functions (e.g. step function) are used. Furthermore, convolution, softmax, max pooling, and contrast normalization, fully connected layer with ReLU, sigmoid, hyperbolic tangent activation functions are proved to be Lipschitz continuous \cite{SzegedyZSBEGF13,RuanHK18}. 
%Volume is an appropriate measure since it is not always 0.

\section{Limitations of Conventional  Robustness} \label{Aanalyses}

The conventional robustness can help to understand the ability of the NN to resist attacks. But it has some limitations in understanding the reliability and building a reliable system in practice. 

In machine learning, a class of hypotheses $\mathcal{H}$ (or possible models $f_\Theta$) is PAC learnable if for any pair ($\varepsilon$, $\delta$) with $0<\varepsilon$, $\delta \leq 0.5$, the learning algorithm produces a high accuracy hypothesis (model) $f_\theta \in \mathcal{H}$ with high probability:
\begin{equation*}
	P(\vert R(f_\theta)-\hat{R}(f_\theta) \vert \leq \varepsilon ) \geq 1-\delta
\end{equation*}
where $R$ is generalization error and $\hat{R}$ is training error. Even a concept is PAC learnable, we cannot expect an algorithm to learn it exactly. The only realistic expectation of a good learner is that it will learn a close approximation to the target concept with high probability. Moreover, we cannot always expect a learner to learn a close approximation to the target concept since sometimes the training set does not represent unseen examples. There may be many concepts consistent with the available data, and unseen examples may have any label. On-manifold robustness is essentially generalization. We cannot eliminate the generalization error and adversarial examples can be everywhere. These adversarial examples make almost all the perfect safe regions in norm bounds too small. Even after improving, they still cannot meet the needs of real-world perturbations. In addition, the improvement is usually accompanied by a sharp drop in accuracy. So only considering robustness with error-free sometimes can be too severe for many applications.

\begin{figure}
		\begin{subfigure}{0.17\textwidth} 
			\includegraphics[width=\linewidth]{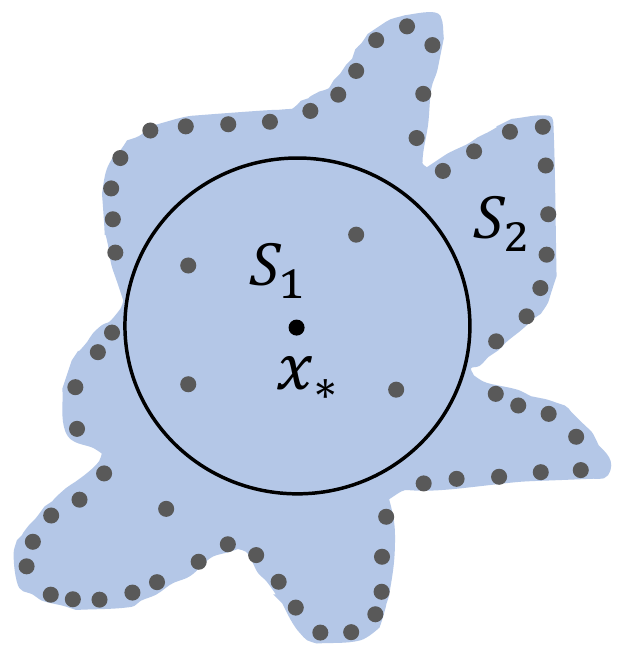}
			\caption{ }
			\label{RobustA}
		\end{subfigure}
		\hspace{10mm}
			\begin{subfigure}{0.159\textwidth} 
				\includegraphics[width=\linewidth]{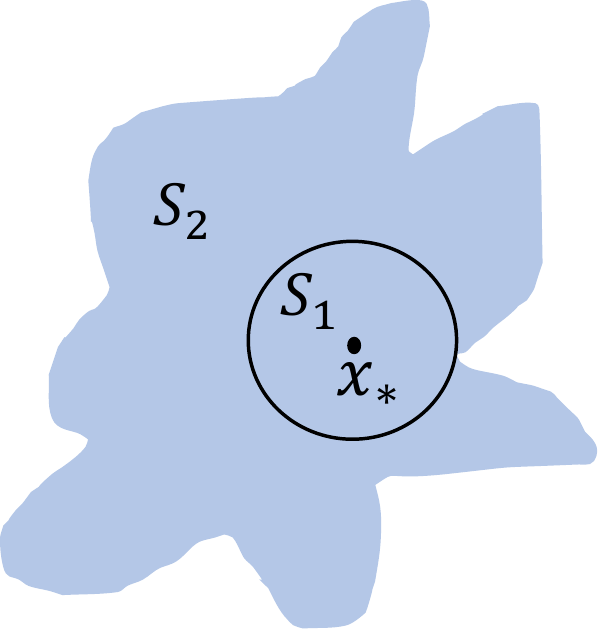}
				\caption{ }
				\label{RobustB}
			\end{subfigure}
				\caption{(a) Situation 1: the majority part in safe region is out of the safe norm ball because of the ``curse of dimensionality''. (b) Situation 2: the safe norm ball is very small as the given point is located near the boundary of the safe region.}
				\label{fig:Robustness}
\end{figure}

Besides, we notice that the definition of conventional robustness can suffer from the ``curse of dimensionality'' and is strongly related to the location of the given test point. 

For example, Figure \ref{fig:Robustness}(a) shows the optimal safe norm ball $B_2(x_*,r)$ around point $x_*$. Suppose the real safe region around the point $x_*$ is $S_2$, then we have:
\begin{equation*}
\lim\limits_{n \rightarrow \infty} \frac{Vol(S_2)-Vol(S_1)}{Vol(S_2)}=1
\end{equation*}
where $n$ is the dimension of the space and $Vol(\cdot)$ is the volume. It implies that the majority of points are located at the edge of $S_2$. This means that as the dimensionality grows $S_1$ is like a ``hollow ball''. Actually, computing the optimal norm ball is very difficult, and existing methods can only give lower bounds. We cannot expect to evaluate $S_2$ accurately, because it is usually a high-dimensional non-convex region. 

Sometimes, a given point may be near the edge of real-world data distribution in that region. Figure \ref{fig:Robustness} (b) shows the situation where given $x_*$ is located in a ``bad'' position. The robustness analysis result reveals that the safe region is very small. However, the real safe region can be much bigger than the safe norm ball. In such a situation, we may draw a biased evaluation result about the safe region because the analysis results are only based on the radius of the safe norm ball. 

\section{ $\varepsilon$-Weakened Robustness}

First, we define an indicative function $T(F,x,\Omega)$. $\Omega$ is designed for multi-labels extension.
\begin{definition}
	Consider a neural classifier $F(x)$. Given a test point $x$ with a label set $\Omega$, the indicative function $T(F,x,\Omega)$ is defined as:
	\begin{equation}
		T(F,x,\Omega)=\left\{
		\begin{aligned}
			1 &  & F(x) \in \Omega\\
			0 &  &F(x) \notin \Omega
		\end{aligned}
		\right.
	\end{equation}
\end{definition}

\begin{definition}[\textbf{$\varepsilon$-weakened robustness}]
	Consider a neural classifier $F(x)$. Given a test point $x_*$ with the gold label $l_*$ and a $p$-norm ball with radius $r$, if the formula:
	\begin{equation}\label{ewr}
		\frac{\mu(\{x |  x \in B_p(x_*, r) \wedge T(F,x,\Omega)=1\})}{\mu(B_p(x_*, r))}> 1-\varepsilon
	\end{equation}
	holds for a measure $\mu$ and parameter $\varepsilon$, then $F(x)$ is said to be $\varepsilon$-weakened robust at point $x_*$ with respect to the radius $r$. $\mu(\cdot) \in [0,+\infty)$, $\Omega=\{l_*\}$, $p \in \{1,2,\infty\}$ and $\varepsilon \in (0,1)$ is a fixed-length computer representable binary number (e.g., float 64 or 32).
\end{definition}
We briefly denote it as \textbf{$\varepsilon$-robustness}. It can be seen as a relaxation of formula (\ref{rob}). Deciding whether formula (\ref{ewr}) is ``SAT'' is called  \textbf{$\varepsilon$-weakened robustness decision problem}. If the input space $\mathbb{D}^n$ is defined on a countable set, $\mu$ is the counting measure (i.e., the cardinality of the set). If the input space is $\mathbb{R}^n$, $\mu$ is the volume. Based on the property depicted in lemma \ref{lemma1}, volume is a reasonable measure. 

\begin{figure}
	\centering
	\includegraphics[width=1\linewidth]{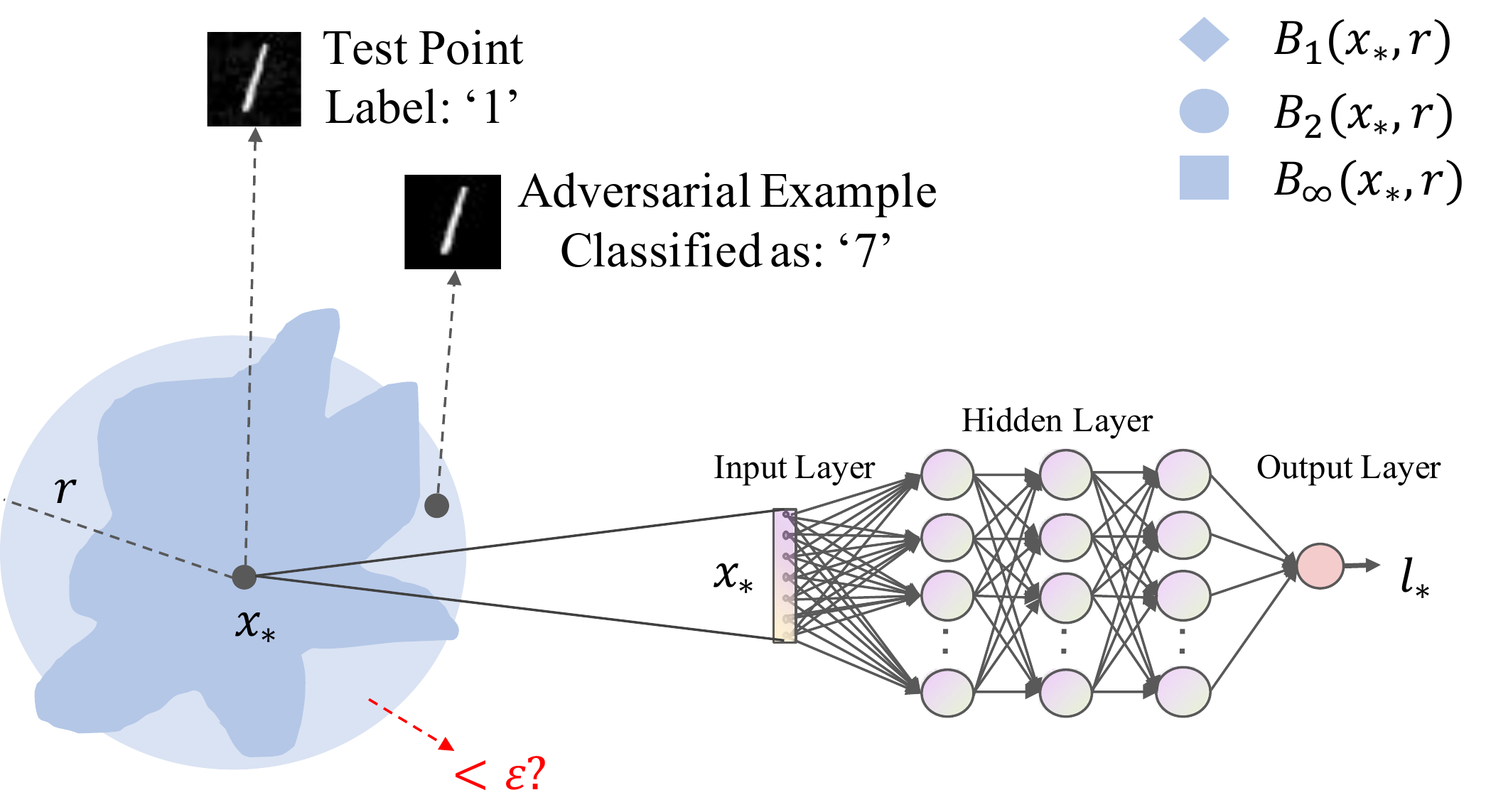}
	\caption{$\varepsilon$-weakened robustness}
	\label{fig:E-robustness}
\end{figure}

\begin{definition} \textbf{$\varepsilon$-weakened robustness evaluation problem} is an optimization problem:
		\begin{equation}
		\begin{split}
		&\mathbf{max} \ \  r  \\
		&\mathbf{s.t}\ \ \varepsilon\text{-}robust(x_*,r)
		\end{split}
		\end{equation}
	where $\varepsilon$-$robust(x_*,r)$ denotes the classifier $F(x)$ is $\varepsilon$-robust at point $x_*$ w.r.t. the radius $r$.
\end{definition}

After introducing $\varepsilon$, we can provide conclusive analysis results about the regions which are not ``perfect'' safe. Besides, the majority of a safe region around a given point is considered. Introducing $\varepsilon$ also helps to alleviate the dependence of analysis results on the locations of given points as the norm bound can go outside the ``perfect'' safe region.

\subsection{Computational Complexity}
Actually, it is better to compute $\mu(\{x |  x \in B_p(x_*, r) \wedge T(F,x,\Omega)=1\})$. However, this problem is \#P-hard and even the Monte Carlo approximation is exponential in the dimensionality for the worst case when the error bound is fixed \cite{simonovits2003compute}. As an alternative, answering whether the formula (\ref{ewr}) is true is computationally easier.

%If $T(F,x,\Omega)=1$ is a combination of convex polytopes, we can estimate it in the framwork of \#SMT(LA) \cite{MaLZ09} with complexity $\mathcal{O}(n^4)$ \cite{lovasz2006, GeM15}.

\begin{theorem}
	$\varepsilon$-robustness decision problem is $PP$-complete.
\end{theorem}
We only show a proof skeleton here and full details are provided in the appendix (see supplementary material). First, we prove that the $\varepsilon$-robustness decision problem is in $PP$. We give a polynomial-time probabilistic algorithm with an error probability of less than 1/2, which is a simplified version of Algorithm 1 (In sec. \ref{Dsec}). Then we prove MAJSAT($\geq2^{n-1}$), which asks if a CNF (or SAT) formula has at least half of the possible truth assignments, can be reduced to $\varepsilon$-robustness decision problem. The MAJSAT($\geq2^{n-1}$) is known to be PP-complete\cite{stoc/Gill74}. We use polynomial-size neurons with ReLU function to simulate negation operators and clauses. In this context, MAJSAT can be transformed to the $1/2$-robustness decision problem of a binary classification network.

\section{Methodology}\label{theorysec}

It is known that $NP \subseteq PP \subseteq \#P$. Although introducing $\varepsilon$ helps to reduce computational complexity, the $PP$-completeness of the problem still means that we must expect the worst-case performance of the exact algorithm to be poor. So we provide a statistical inference method to tackle the $\varepsilon$-robustness decision problem with user-controllable error bounds (or confidence). In most cases, such a provable guarantee can meet the needs of software engineering.

We put the $\varepsilon$-robustness decision problem in the framework of hypothesis testing, and explain how to control the error bounds.
\subsection{Reduction to hypothesis testing}
Let 
\begin{equation*}
	p_{r}=\frac{\mu(\{x' |  x' \in B_p(x_*, r) \wedge T(F,x',\Omega)=1\})}{\mu(B_p(x_*, r))}\
\end{equation*}
then the $\varepsilon$-robustness decision problem is to determine whether $H_0$ or $H_1$ is true:
\begin{equation*}
	H_0:p_{r} > 1-\varepsilon \ \ \ \ H_1:p_{r} \leq 1-\varepsilon
\end{equation*}
$H_0$ is called \textbf{null hypothesis} which means the NN is $\varepsilon$-robust and $H_1$ is called \textbf{alternative hypothesis} which is the opposite.

A procedure which is based on the value of a statistic, usually the sample mean $\overline{y}$, for deciding whether to accept or reject $H_0$ is called a \emph{test of a statistical hypothesis}. It is also a method for measuring the strength of the evidence against the null hypothesis contained in the experimental data. The set of values 
\begin{equation*}
	\mathcal{C} = \{ \overline{y}: y \leq c\}
\end{equation*} 
is called the \textbf{rejection region} of the test, since $H_0$ is rejected when $\overline{y}$ falls in $\mathcal{C}$. Two types of errors can occur in this statistical decision-making process, and they are shown in Table \ref{altable}.
\begin{table}[h]
	\caption{Null and alternative hypothesis, decision rules, type I and type II errors.}
	\label{altable}
	\begin{tabular}{|c|c|l|}
		\cline{1-3}
		$\overline{y}$&$H_0\ is\ true$ & $H_1\ is\ true$\\
		\cline{1-3}
		$\overline{y}> c$ (accept $H_0$)&$Correct\ decision$& $Type\ \uppercase\expandafter{\romannumeral2}\ error$ \\
		\cline{1-3}
		$\overline{y}\leq c$ (reject $H_0$)& $Type\ \uppercase\expandafter{\romannumeral1}\ error$  &$Correct\ decision$\\
		\cline{1-3}
	\end{tabular}
\end{table}

A type $\uppercase\expandafter{\romannumeral1}$ error occurs when $H_0$ is actually true but one rejects $H_0$ (accepts $H_1$). A type $\uppercase\expandafter{\romannumeral2}$ error occurs when $H_0$ is actually false but one accepts $H_0$ (rejects $H_1$). 
\begin{itemize}
	\item The probability of type $\uppercase\expandafter{\romannumeral1}$ error (false positive) is defined as:
	\begin{equation*}
		P(\text{Reject } H_0|H_0\text{ is true})=P(\overline{Y}\leq c |  H_0\text{ is true})
	\end{equation*} 
	\item The probability of type $\uppercase\expandafter{\romannumeral2}$ error (false negative) is defined as:
	\begin{equation*}
		P(\text{Accept } H_0| H_0 \text{ is false})=P(\overline{Y}> c |  H_0 \text{ is false})
	\end{equation*} 
\end{itemize}
\textbf{Terminology:} The significance level of the test is the probability of a type \uppercase\expandafter{\romannumeral1} error.

We can answer the hypothesis $H_0$ with uniform sampling in $B_p(x_*, r)$. We know that the satisfaction of $T(F,x_*,\Omega)=1$ is binomial distribution. We take random samples of sizes $N$ from the space $B_p(x_*, r)$ independently, and denote their possible outcomes as independent random variables $Y_1, Y_2,...,Y_N$ where $Y_i \sim Bernoulli(p_r)$.

If $N$ satisfies
\begin{equation*}
	N>N(1-\varepsilon)+3\sqrt{N\varepsilon(1-\varepsilon)}
\end{equation*} 
\begin{equation}\label{samplesize1}   
	\text{or } \sqrt N >3 \sqrt \frac{1-\varepsilon}{\varepsilon}
\end{equation} 
then a ``large-sample'' test is done. Based on \textbf{central limit theorem}, if $\overline{Y}=(\sum\nolimits_{i=1}^N  Y_i)/N$, $\sigma=p_r(1-p_r)$ we have:
\begin{equation}
	Z=\frac{(\overline{Y}-p_r)}{\sigma/\sqrt N} \sim  \mathcal{N}(0,1)
\end{equation} 
where $\mathcal{N}(0,1)$ is a Gaussian distribution.

\subsection{Controlling Type \uppercase\expandafter{\romannumeral2} Error}
We first introduce how to control the probability of committing a type $\uppercase\expandafter{\romannumeral2}$ error. Because it is very important for answering the $\varepsilon$-robustness decision problem. We do not hope the network is not $\varepsilon$-robust but the algorithm often answers ``SAT''.

The probability of a type $\uppercase\expandafter{\romannumeral2}$ error is as follows:
\begin{equation*}
	\begin{aligned} 
		P(\overline{Y} > c |  H_0\text{ is false})& \leq P(\overline{Y}> c | p_r=1-\varepsilon)\\
		&=P(\frac{\overline{Y}-p_r}{\sigma/\sqrt N} > \frac{c-p_r}{\sigma/\sqrt N} | p_r=1-\varepsilon)\\
		&=P(\frac{\overline{Y}-p_r}{\sigma/\sqrt N}> \frac{c-(1-\varepsilon)}{\sigma/\sqrt N})\\
		&=P(Z > \frac{c-(1-\varepsilon)}{\varepsilon(1-\varepsilon)/\sqrt N})
	\end{aligned}
\end{equation*}

So, if we expect the probability of committing a type $\uppercase\expandafter{\romannumeral2}$ error is less than $\beta$, we have :
\begin{equation*}
	P(Z > \frac{c-(1-\varepsilon)}{\varepsilon(1-\varepsilon)/\sqrt N}) \leq \beta
\end{equation*}
\begin{equation*}
	\Rightarrow \frac{c-(1-\varepsilon)}{\varepsilon(1-\varepsilon)/\sqrt N}\geq z_{1-\beta}
\end{equation*}

$z_{1-\beta}$ is the $1-\beta$ quantile of $\mathcal{N}(0,1)$ and we have:
\begin{equation}\label{con_2error}
	c \geq \frac{\varepsilon(1-\varepsilon)z_{1-\beta}}{\sqrt N}+(1-\varepsilon)
\end{equation}
When $c$ satisfies formula (\ref{con_2error}), the probability of committing a type $\uppercase\expandafter{\romannumeral2}$ error is smaller than $\beta$.

\subsection{Controlling Type \uppercase\expandafter{\romannumeral1} Error}
The probability of type $\uppercase\expandafter{\romannumeral1}$ error follows:
\begin{equation*}
	\begin{aligned}
		P(\overline{Y}\leq c |  H_0 \text{ is true})&=P(\overline{Y}\leq c |  p_r > 1-\varepsilon)\\
		&=P(\overline{Y}\leq c |  p_r \geq 1-\varepsilon')\\
		& \leq P(\overline{Y}\leq c |  p_r = 1-\varepsilon')\\
		&=P(\frac{\overline{Y}-p_r}{\sigma/\sqrt N}\leq \frac{c-p_r}{\sigma/\sqrt N} | p_r=1-\varepsilon')\\
		&=P(Z \leq \frac{c-(1-\varepsilon')}{\sigma/\sqrt N})\\
	\end{aligned}
\end{equation*}
$\varepsilon'$ is the closest machine number where $\varepsilon>\varepsilon'$.

So, the probability of committing a type $\uppercase\expandafter{\romannumeral1}$ error being less than $\alpha$ can be given by:
\begin{equation*}
	\begin{aligned} 
		&P(Z \leq \frac{c-(1-\varepsilon')}{\sigma/\sqrt N}) \leq \alpha\\
		&\Rightarrow \frac{c-(1-\varepsilon')}{\varepsilon'(1-\varepsilon')/\sqrt N}\leq z_{\alpha}\\
	\end{aligned}
\end{equation*}
where $z_{\alpha}$ is the $\alpha$ quantile of $\mathcal{N}(0,1)$. 

Let $c=(\varepsilon(1-\varepsilon)z_{1-\beta})/\sqrt N+(1-\varepsilon)$, we have:
\begin{equation}\label{samplesize2}
	\sqrt N \geq \frac{\varepsilon(1-\varepsilon)z_{1-\beta}-\varepsilon'(1-\varepsilon')z_{\alpha}}{\varepsilon-\varepsilon'}
\end{equation}
Based on (\ref{samplesize1}) and (\ref{samplesize2}), the sample size $N$ is given by:
\begin{equation}\label{samplesize3}
	\sqrt N \geq \max\{\frac{\varepsilon(1-\varepsilon)z_{1-\beta}-\varepsilon'(1-\varepsilon')z_{\alpha}}{\varepsilon-\varepsilon'},3 \sqrt \frac{1-\varepsilon}{\varepsilon}\}
\end{equation}

If the sample size $N$ satisfies formula (\ref{samplesize3}), then the probability of committing a type $\uppercase\expandafter{\romannumeral1}$ error is smaller than $\alpha$. It is also called a test with \textbf{significant level at $\alpha$}. 

In practice, controlling the type \uppercase\expandafter{\romannumeral1} error and type \uppercase\expandafter{\romannumeral2} error simultaneously can be a heavy burden for computation. In order to improve efficiency, we strictly control \uppercase\expandafter{\romannumeral2} error and loosen the other one as an alternative. For robustness decision problem, strictly controlling type \uppercase\expandafter{\romannumeral2} error is more important. Because when the network is not robust, the algorithm answering ``SAT'' can be dangerous. So we let $\varepsilon'$ smaller than $\varepsilon$ e.g., $\varepsilon'=\varepsilon-\min(\sigma, 0.005)$ where $\sigma \approx \varepsilon(1-\varepsilon)$. It means we strictly control $P(\text{reject } H_0 |  p_r \geq 1-\varepsilon')$ to be less than $\alpha$.

\section{Algorithms}
In this section, we present a probabilistic algorithm for $\varepsilon$-robustness decision problem and analyze the time complexity. Then, we give an algorithm to find the maximum $\varepsilon$-robustness radius.
\subsection{$\varepsilon$-Robustness Decision Algorithm}\label{Dsec}
Based on the theory described in Sec.\ref{theorysec}, the decision algorithm is outlined in Algorithm 1 and the sampling methods are in appendix.
\begin{algorithm}
	\caption{Probabilistic algorithm for $\varepsilon$-weakening robustness decision problem (Is-$\varepsilon$-robust())}
	\KwIn{$F$- network classifer, $x_*$- test point, $r$-distortion radius, $p$- $p$ norm $\varepsilon$- weakening parameter, $\alpha$-type $\uppercase\expandafter{\romannumeral1}$ error , $\beta$-type $\uppercase\expandafter{\romannumeral2}$ error}
	\KwOut{ $SAT$ or $UNSAT$.}
	$\Omega \leftarrow \{l_*-\text{label of }x_*\}$\;
	$Construct\ function\ T(F, x, \Omega)$\;
	$z_{\alpha}\leftarrow \Phi^{-1}(\alpha)$ ; \tcc*[f]{$\Phi$ is the CDF of $\mathcal{N}(0,1)$}
	
	$z_{1-\beta}\leftarrow \Phi^{-1}(1-\beta)$\;
	$\varepsilon'=\varepsilon-\min(\varepsilon(1-\varepsilon), 0.005)$\;
	$N \leftarrow \lceil\max\{(\frac{\varepsilon(1-\varepsilon)z_{1-\beta}-\varepsilon'(1-\varepsilon')z_{\alpha}}{\varepsilon-\varepsilon'})^2, 9\frac{1-\varepsilon}{\varepsilon}\}\rceil$\;
	$c \leftarrow \frac{\varepsilon(1-\varepsilon)z_{1-\beta}}{\sqrt N}+(1-\varepsilon)$\; 
	$\overline{y}\leftarrow 0$\;
	\For{$i \leftarrow 1$ \KwTo $N$}
	{
		$y_{i} \leftarrow sample\ a\ point\ uniformly\ in\ B_p(x_*,r)$ \;
		\If{$T(F, y_{i}, \Omega)=1$}{
			$\overline{y}\leftarrow \overline{y}+1$\;
			
		}
					\lIf{$\overline{y} \geq c \cdot N$}{\Return $SAT$}
					\lIf{$\overline{y} < (c-1) \cdot N+i$}{\Return $UNSAT$}	
	}
\end{algorithm}

The input parameter $\varepsilon$ can be specified by the safety requirements. The probabilities of false positive and false negative,$\alpha$ and $\beta$, are controllable to users. In theory, $p$ can take any positive integer, but it is difficult to give an efficient uniform sampling algorithm when dimension $n$ is large. The time complexity of the naive rejection sampling method is exponential in dimension $n$. $\varepsilon'$ can be any value in $(0,\varepsilon)$ customized by the balance between precision requirements and running time. 

The time complexity of uniformly sampling in $\ell_1$-norm ball is $\mathcal{O}(n\log n)$ in average and $\mathcal{O}(n^2)$ in the worst case. Each component of a sampling point is constructed by mutually independent random variables over $[0,r]$. The time complexity of uniformly sampling in $\ell_2$-norm ball is $\mathcal{O}(n)$. It was first proposed by Muller \cite{Muller}. The algorithm makes use of a strong property of the Gaussian distribution that the exponent part is the same as the expression for calculating the $\ell_2$-norm distance. The time complexity of uniformly sampling in $\ell_{\infty}$-norm ball is $\mathcal{O}(n)$.

We use $t_s(n)$ to denote the running time of sampling algorithm and $t(|F|)$ to denote the execution time of NN for an input. $t(|F|)$ is a polynomial function of the size of network. So the time complexity of Algorithm 1 is $\mathcal{O}(N\cdot(t_s(n)+t(|F|)))$. Based on statistical evidence, we can infer the final result before sampling $N$ times (line 14-15 in Algorithm 1). Thus Algorithm 1 can take less time than the theoretical upper bound.

Next, we analyze the influence of the parameters on Algorithm 1.
\begin{itemize}
	\item Based on Eq. (\ref{samplesize1}), when $\varepsilon \rightarrow 0$ the sample size $N \rightarrow \infty$. $N$ is a polynomial function of $1/\varepsilon$.

	\item When $\varepsilon' \rightarrow \varepsilon$, the sample size $N\rightarrow \infty$ (line 6 of the Algorithm 1 or Eq. (\ref{samplesize2})). In practical, when $(\varepsilon-\varepsilon')=$$\sigma$ or $0.005$, the number of samples is about tens of thousands.
	
	\item When $\beta \rightarrow 0$, $z_{1-\beta}\rightarrow \infty$. Based on Eq. (\ref{samplesize2}), we have the sample size $N \rightarrow \infty$. The smaller the probability of committing a type $\uppercase\expandafter{\romannumeral2}$ error we expect, the more sampling points we need, and we cannot eliminate the error. The relationship between $\beta$ and $N$ cannot be expressed by an elementary function, but the increase of $N$ is relatively flat because of the good properties of the Gaussian distribution.
	
	\item When $\alpha \rightarrow 0$, $z_{\alpha}\rightarrow -\infty$. Based on Eq. (\ref{samplesize2}), we have the sample size $N \rightarrow \infty$. It is similar to $\beta$.
\end{itemize}
The implementation of our algorithm is parallel on the GPUs.

\subsection{$\varepsilon$-Robustness Evaluation Algorithm}
In this subsection, we present an algorithm for $\varepsilon$-weakened robustness evaluation problem. It can be tackled on the framework of Turing reduction. In sec \ref{Dsec}, we have shown how to decide whether a neural network is $\varepsilon$-robust for a given radius $r$. So the maximum value of $r$ can be searched by querying this decision oracle iteratively. This framework can be used to solve many optimization problems in theory. However, whether it works in practice depends on the complexity of the decision oracle. 

It can be realized in incremental or decremental search but they are not efficient in practice. We assume that the number of adversarial examples will increase as the radius increases. Although it has not been rigorously proven, it is usually the case. Based on this assumption, we can realize an efficient binary search version such as Algorithm 5. 
\begin{algorithm}[h]
	\caption{Algorithm for $\varepsilon$-weakening robustness evaluating problem}
	\KwIn{$F$- neural classifer, $x_*$- test point, $R$-upper bound for distortion radius, $p$- $p$ norm $\varepsilon$- weakening parameter, $\alpha$-type $\uppercase\expandafter{\romannumeral1}$ error , $\beta$-type $\uppercase\expandafter{\romannumeral2}$ error}
	\KwOut{$r_*$}
	$r_{min} \leftarrow 0$\;
	$r_{max} \leftarrow R$\;
	\While{$r_{max}-r_{min}>precision$ and $r_{max}\neq0$ and $r_{min}\neq R$}{
		$r=(r_{max}-r_{min})/2$\;
		\eIf{Is-$\varepsilon$-robust$(F,x_*,r,p,\varepsilon,\alpha,\beta)=SAT$}
		{$r_{min} \leftarrow r$\;}
		{$r_{max} \leftarrow r$\;}
	}
	{\Return $r_*\leftarrow r_{min}$}
\end{algorithm}

\section{Multi-label Extension}
Sometimes, the misclassification of a neural network does not necessarily lead to disaster. For example, in a self-driving system, if the neural network recognizes a car as a truck, this situation may not be dangerous. Another example is that if the neural network recognizes an 80 km speed limit sign as a 60 km one, the vehicle is safe at this time. Our evaluating method can be compatible with this situation. We just need to add multiple labels to the set $\Omega$.
\begin{example}
	A point $x_*$ is a car and $l_0$ is its label. We denote the label of trucks as $l_1$ and minivans as $l_2$. Then, let $\Omega=\{l_0,l_1,l_2\}$, we can evaluate the $\varepsilon$-weakening robustness of recognizing the possible vehicles.
\end{example}

This extension is inapplicable for conventional robustness, because the regions of different predictions may not be contiguous. So the robustness radius can be no change compared with the single label one.

%\section{Potantial Application}
%Define the ``bug'' of neural network.
\section{Experiments}
We implemented our algorithms with Python3. We use the PyTorch deep learning platform to train and test DNNs. The hardware setups are as follows. CPU: Intel Xeon Gold 6154 @3.00GHz. GPU: GeForce RTX 2080TI. OS: Ubuntu 18.04.3.
\begin{itemize}
	\item \textbf{EWRD} is the name of our $\varepsilon$-\textbf{W}eakened \textbf{R}obustness \textbf{D}ecision tool\footnote{The link of the tools will be given after the double-blind review.\label{footnote1}}.
	\item \textbf{EWRE} is the name of our $\varepsilon$-\textbf{W}eakened \textbf{R}obustness \textbf{E}valuation tool\textsuperscript{\ref {footnote1}}.
\end{itemize}
\textbf{ERAN}\cite{SinghGPV19} which is the state-of-the-art robustness verification tool for neural networks is used as a reference in our experiment.

MNIST, CIFAR-10, MiniImageNet are used in experiments. The MNIST dataset consists of handwritten digital images with size 28$\times$28 in 10 classes. CIFAR-10 dataset consists of color images with size 32$\times$32$\times$3 in 10 classes. MiniImageNet dataset consists of color images with input size 224$\times$224$\times$3 in 100 classes. The value of each pixel is in $[0,255]$ and the perturbation is considered on the original input images. $\alpha$ and $\beta$ are set to 0.001 by default in all experiments. 

\subsection{$\varepsilon$-Robustness \& Reliability}
\textbf{Q1:} \emph{How about the reliability of NNs under the premise of allowing the existence of some adversarial examples?}

We first compare the $\varepsilon$-robustness with conventional robustness on CIFAR-10 of a convolutional network which has 4 layers and 7154 neurons. The experiment is conducted in the context of $\ell_{\infty}$-perturbations because ERAN does not support other norms. Fig \ref{ERvsEW}(\subref{ERAN:sub1}) shows the robustness curve given by ERAN. Fig \ref{ERvsEW}(\subref{EWRD:sub2}) shows the robustness curve given by our EWRD for various $\varepsilon$. We can see that for most regions, perfect safe radii reported by ERAN are usually smaller than $2$. ERAN uses the abstract interpretation technique, which is an over-approximation method, so the radius that can be verified to be safe is usually smaller than the real one. Based on \cite{WongRK20,ShafahiNG0DSDTG19}, almost all regions of the most popular NNs will be attacked successfully at $r=8$ on CIFAR-10. However, if a small proportion of adversarial examples is allowed, the ``robust'' regions will become much bigger. When perturbation radius is $18.49$, 70\% regions are $0.001$-robust.
 
Then we show the experimental results on four popular DNNs: ResNet18 \cite{HeZRS16}, RegNetX \cite{RadosavovicKGHD20}, DenseNet121 \cite{HuangLMW17} and DPN92 (Dual Path Networks) \cite{ChenLXJYF17}. These networks are difficult for conventional robustness verification tools. So we omit the results about conventional robustness. We apply EWRD to these networks on CIFAR-10 for $\varepsilon \in \{0.01,0.1,0.15,0.2\}$. The robustness curves are shown in Fig \ref{cifa:verified}. Take ResNet18 in Fig \ref{cifa:verified} (\subref{cifa:sub1}) for an example, when perturbation radius $r$ is smaller than $12.75$, more than 90\% test points satisfy $0.01$-robustness. That means most areas (90\%) of the network can perform well against a perturbation within distance $12.75$ and the probability of misclassification is less than $0.01$. However, when we increase perturbation $r$ to $76.5$, there are only 25\% regions satisfying $0.2$-robustness. In other words, most areas (75\%) of the network are vulnerable to such perturbations. They may yield to a perturbation with a probability of more than 20\%. We also find that when the radius is greater than a certain threshold, the number of adversarial examples will increase sharply. So the distance among curves of  $\varepsilon=0.1,0.15,0.2$ is very close.

%When $r\geq4$, the number of verified robust points via ERAN is close to 0. There are two reasons: (1) A strictly safe radius is usually small. (2) Abstract interpretation is an over-approximation method. Based on the result reported by the robustness verifier, this network seems to be not good. However, we know that this network is not so bad, it has a good generalization ability (more than 98\% accuracy on the test set). If we can tolerate some adversarial examples, even the proportion is smaller than $0.001$,  we can see that the percentage of robustness points reaches 97.5\% when $r=50$ which are shown in sub-figure (b) in Figure \ref{ERvsEW}.
%We know that this network are not so bad, it have a good generalization (more than 98\% accuracy on the test set). Why a well-trained neural network is lack of robustness but shows good generalization?
\begin{figure}
	\centering
	\begin{subfigure}{0.23\textwidth}
		\centering   
		\includegraphics[width=1\linewidth]{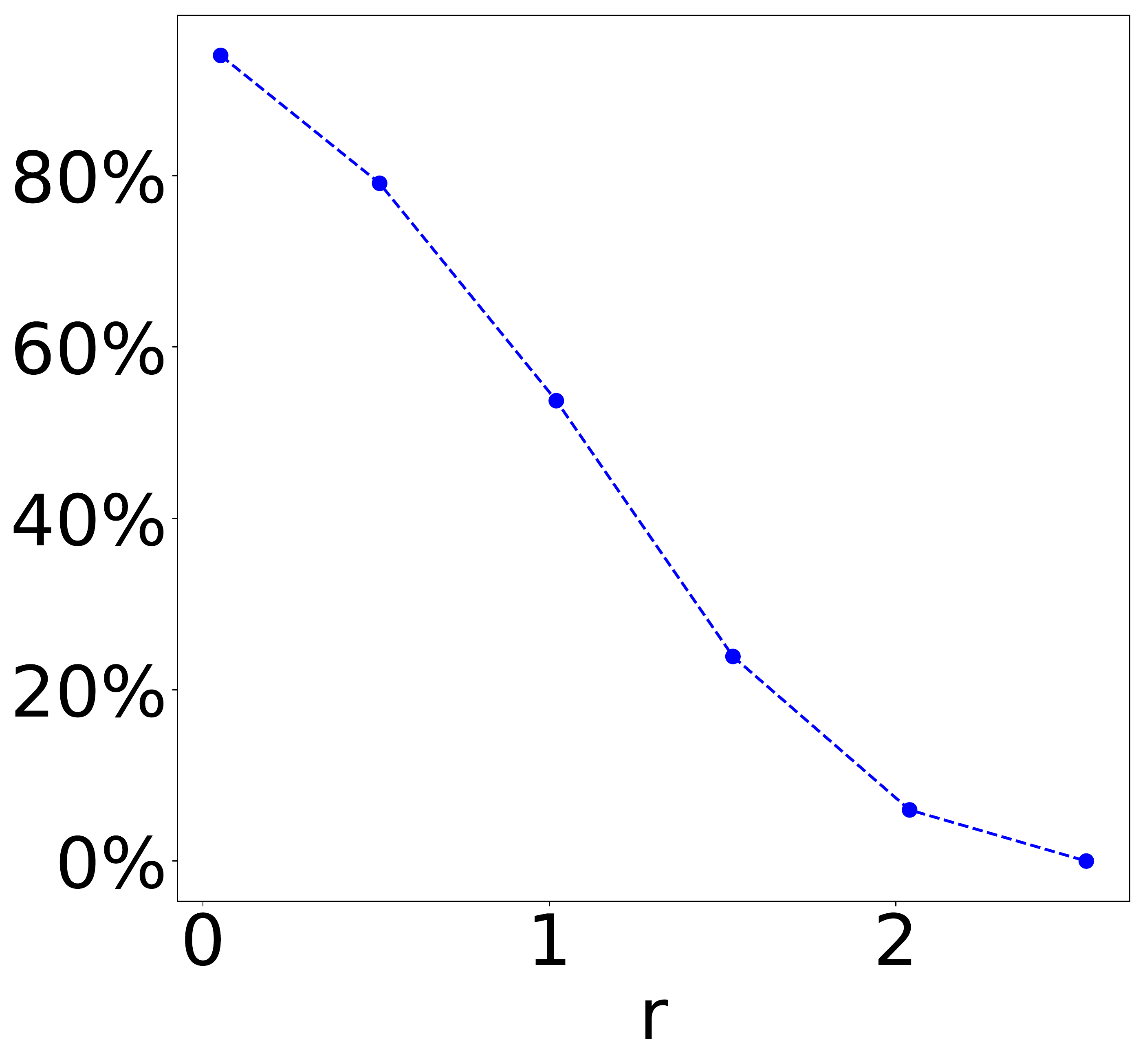}
		\caption{ERAN}
		\label{ERAN:sub1}
	\end{subfigure}
	\begin{subfigure}{0.23\textwidth}
		\centering   
		\includegraphics[width=\linewidth]{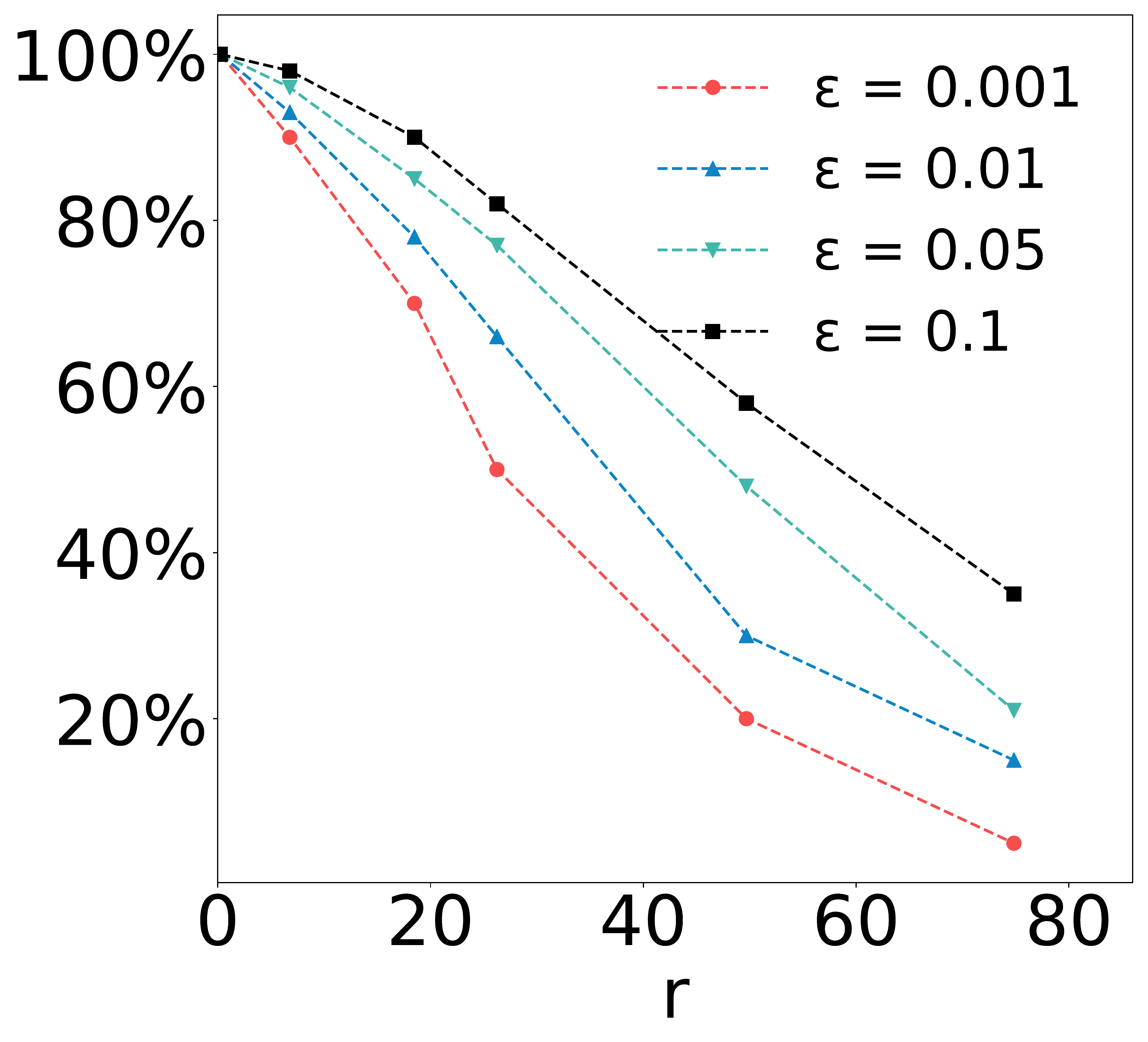}
		\caption{EWRD}
		\label{EWRD:sub2}
	\end{subfigure}	
	\caption{
		Comparing the robustness and $\varepsilon$-robustness curves of the same neural network on CIFAR-10 ($p=\infty$). $x$-axis is the given radius $r$ and $y$-axis is the percent of regions reported to be robust/$\varepsilon$-robust.
	}
	\label{ERvsEW}
\end{figure}

\begin{figure}
	\begin{subfigure}{0.23\textwidth}
		\includegraphics[width=1\linewidth]{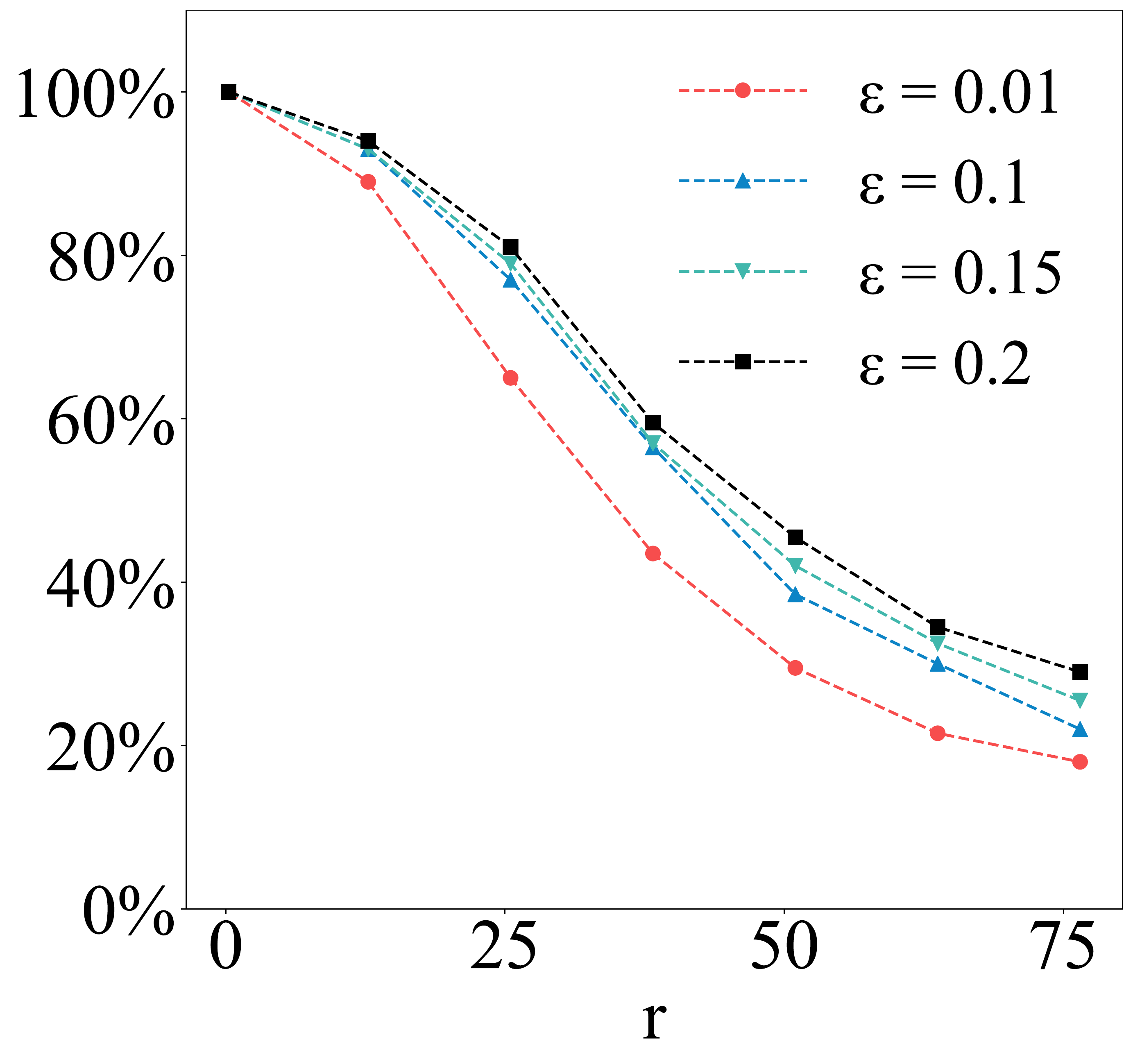}
		\caption{ResNet18}
		\label{cifa:sub1}
	\end{subfigure} 
	\begin{subfigure}{0.23\textwidth}
		\includegraphics[width=\linewidth]{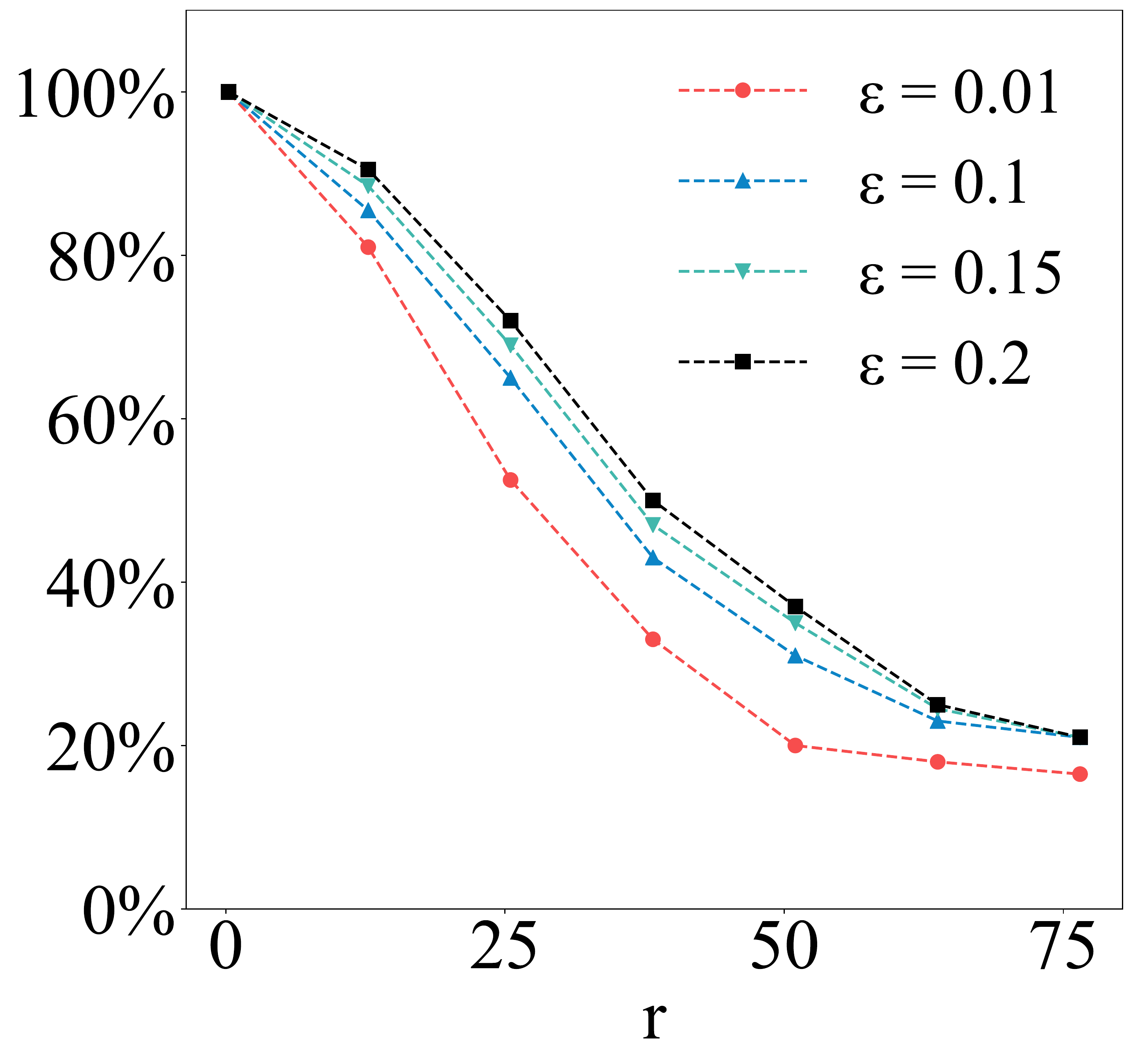}
		\caption{RegNetX}
		\label{cifa:sub2}
	\end{subfigure}
	
	\begin{subfigure}{0.23\textwidth} 
		\includegraphics[width=\linewidth]{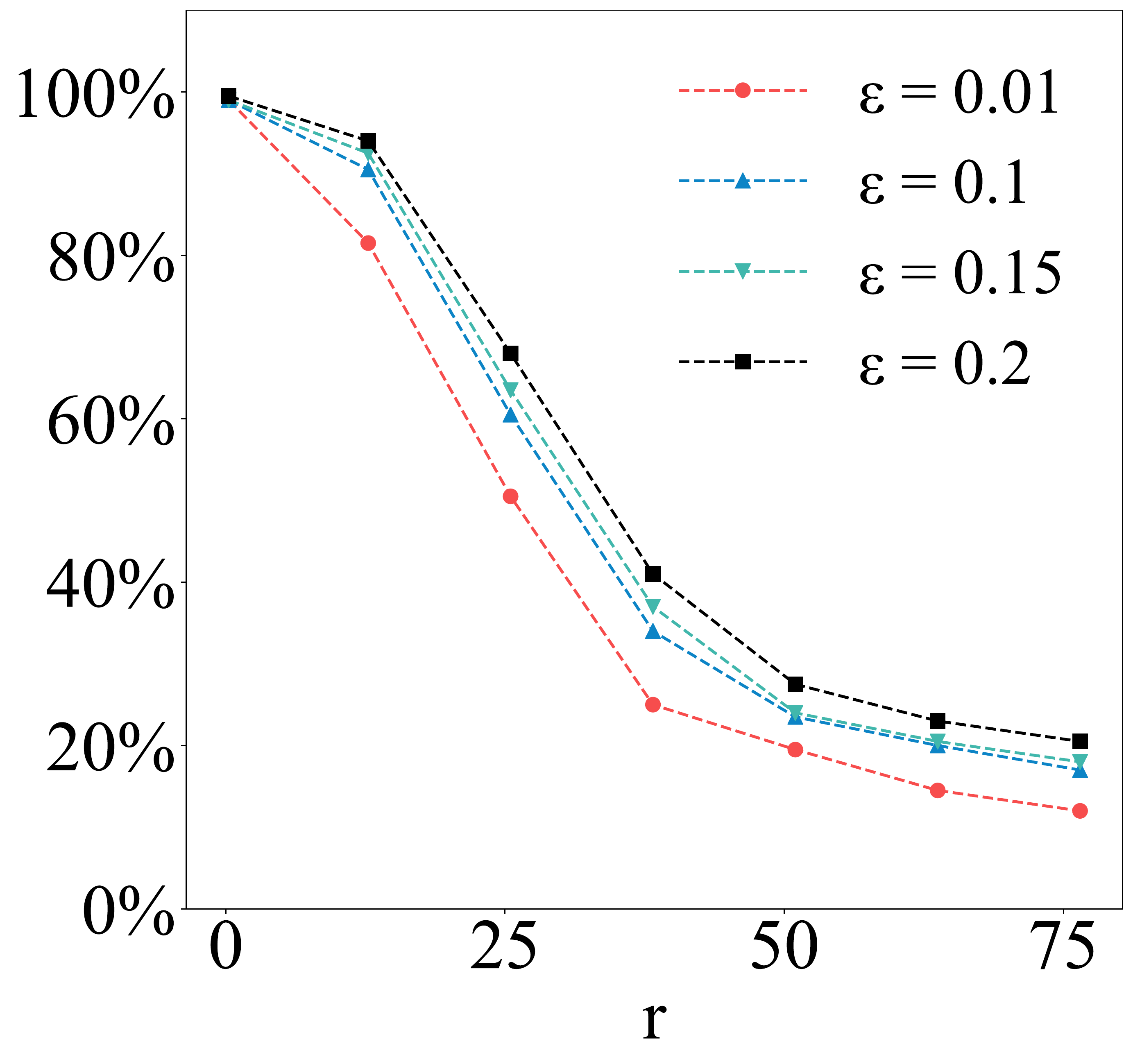}
		\caption{DenseNet121}
		\label{cifa:sub3}
	\end{subfigure}
	\begin{subfigure}{0.23\textwidth}
		\includegraphics[width=\linewidth]{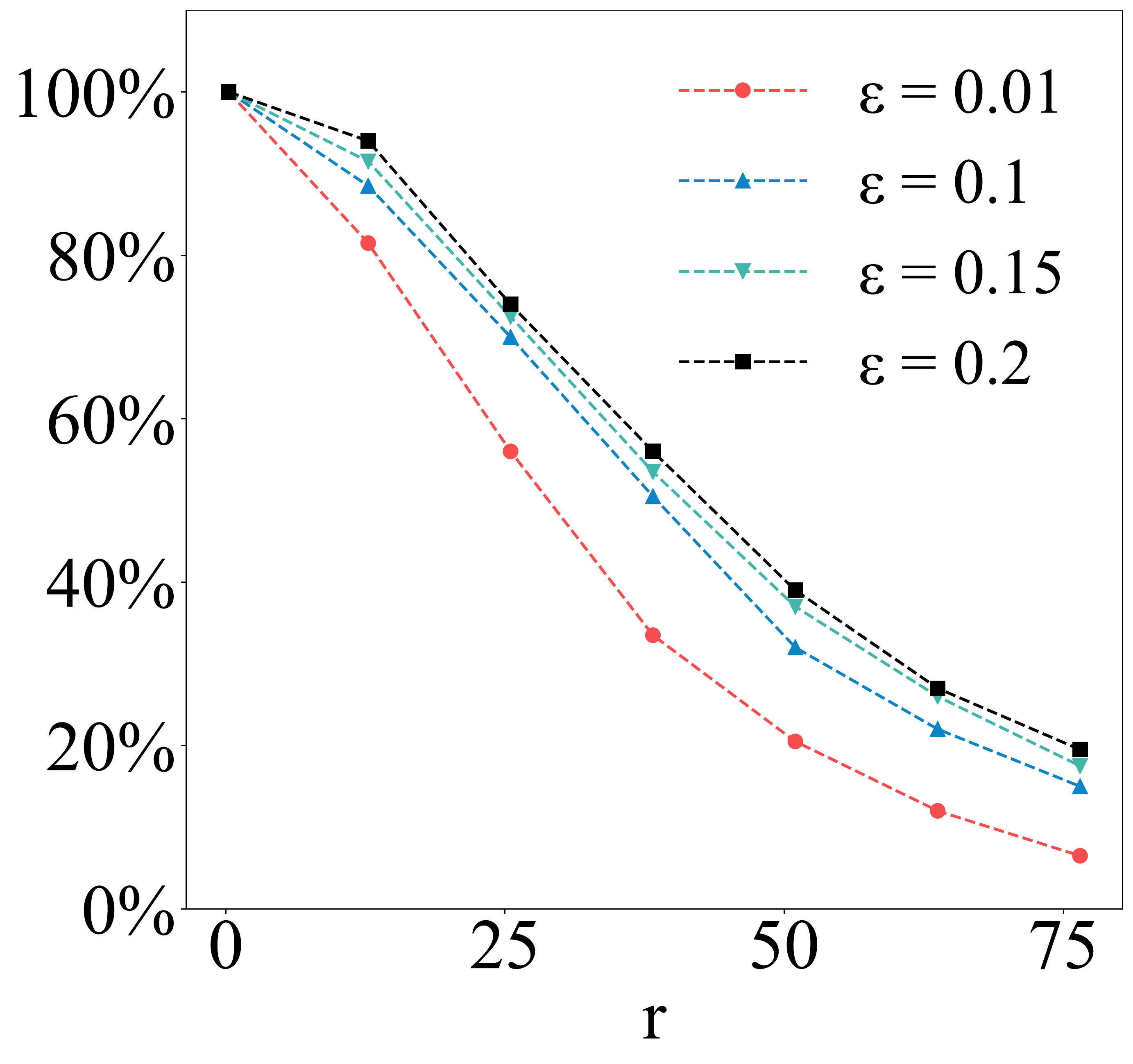}
		\caption{DPN92}
		\label{cifa:sub4}
	\end{subfigure}
	\caption{
		The $\varepsilon$-robustness curves of different neural networks on CIFAR-10 ($p$=$\infty$). $x$-axis is the given radius $r$ and $y$-axis is the percent of regions reported to be $\varepsilon$-robust.}
	\label{cifa:verified}
\end{figure}

We also compare the reliability performance under different settings of norm ($p = \{1,2,\infty\}$) and apply EWRE to pretrained networks in Pytorch (VGG19 \cite{SimonyanZ14a} and AlexNet \cite{KrizhevskySH17} on ImageNet) as shown in Figure \ref{p:total}. With the increase of dimension $n$, we have:
\begin{equation*}
\lim\limits_{n \rightarrow \infty} \frac{Vol(B_1(x_*, r))}{Vol(B_2(x_*, r))}=0\ \ and \\ \lim\limits_{n \rightarrow \infty} \frac{Vol(B_2(x_*, r))}{Vol(B_{\infty}(x_*, r))}=0
\end{equation*}
So, the values of $\varepsilon$-robustness radius decrease greatly as $p$ increases in $\{1,2,\infty\}$. However, their changing trends are almost the same.

\begin{figure}
	\begin{subfigure}{0.155\textwidth} %0.155
		\includegraphics[width=1\linewidth]{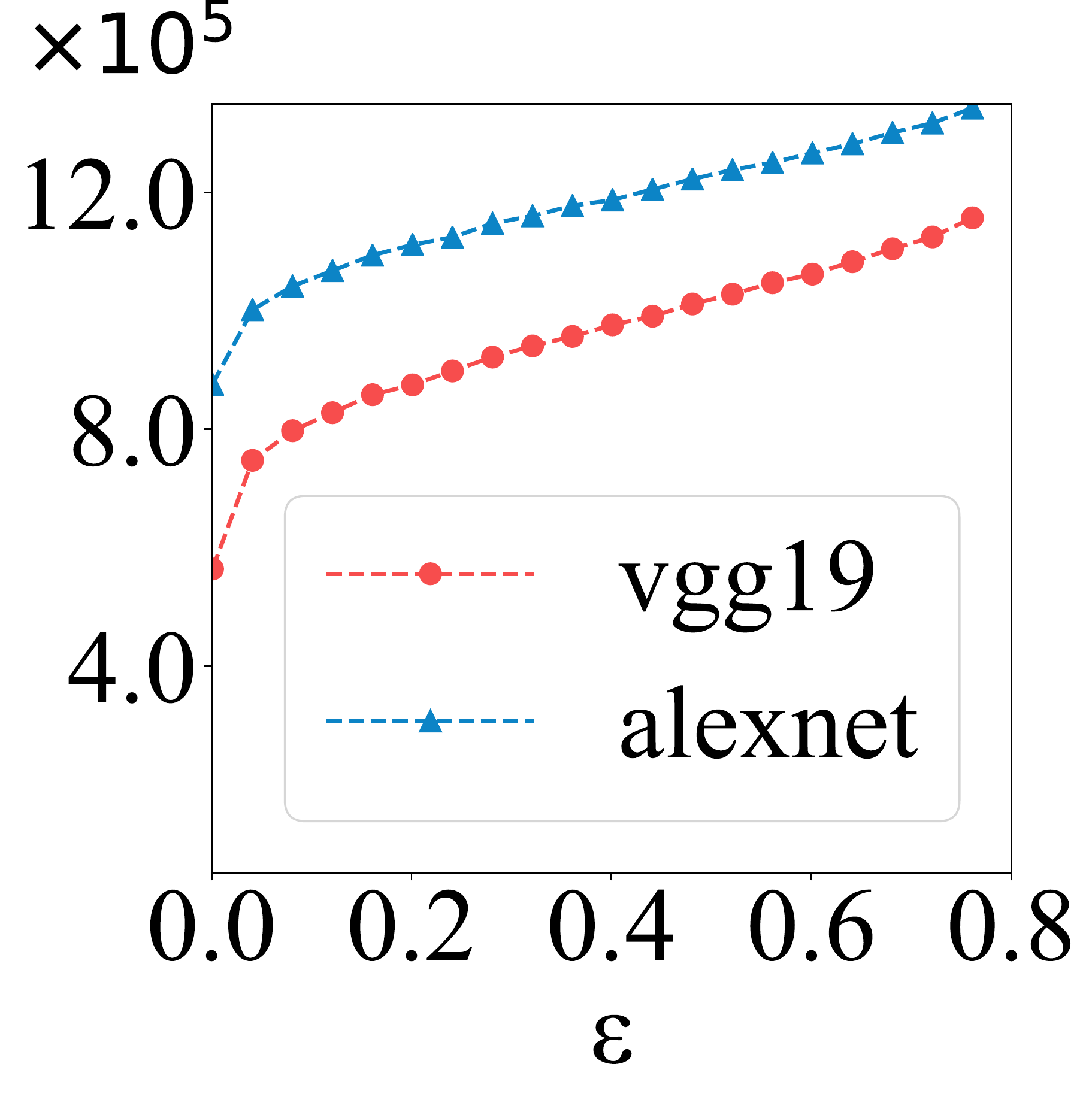}
		\caption{$\ell_{1}\ (p=1)$}
		\label{p1:sub1}
	\end{subfigure}
	\begin{subfigure}{0.155\textwidth} 
		\includegraphics[width=\linewidth]{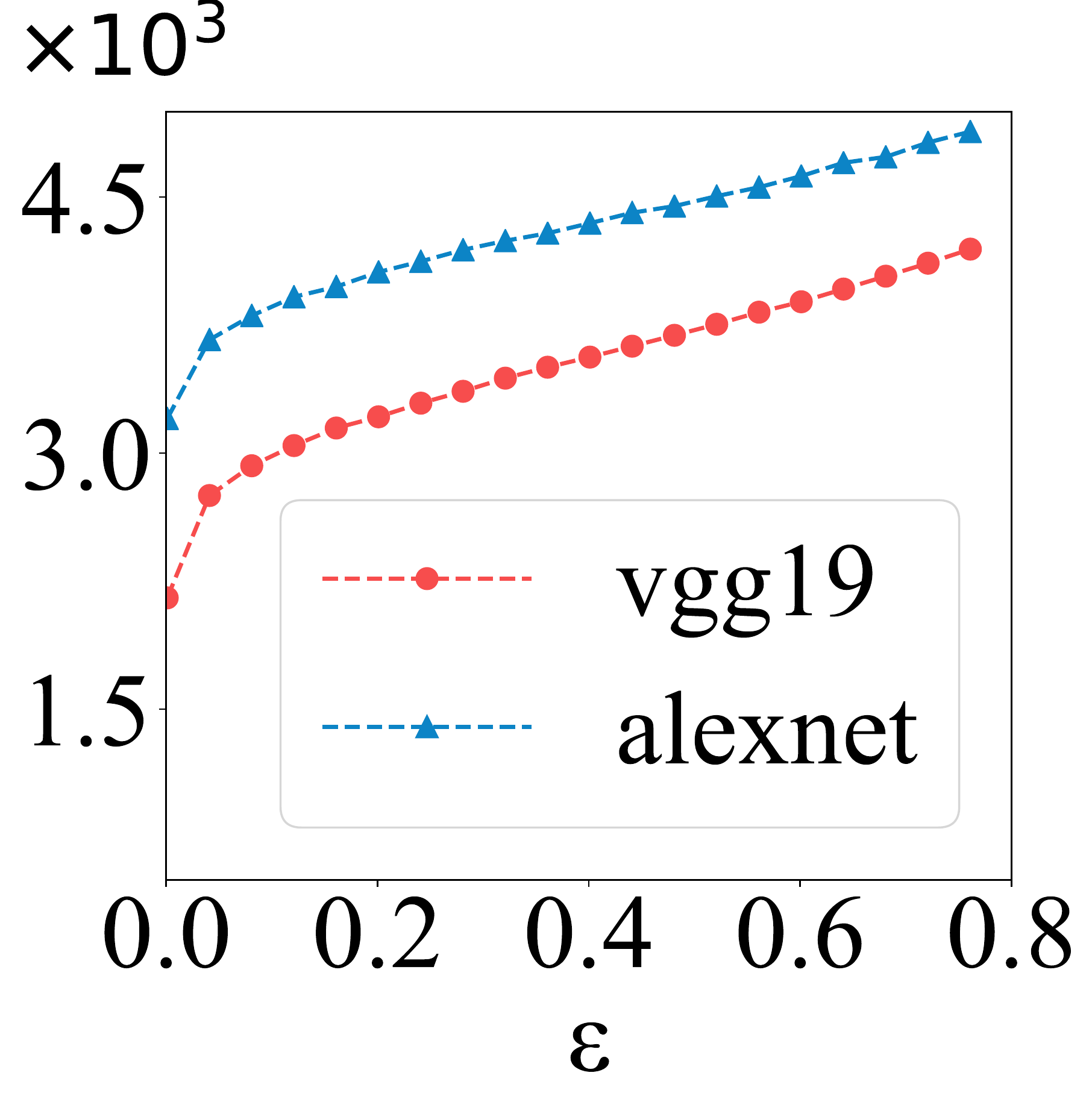}
		\caption{$\ell_{2}\ (p=2)$}
		\label{p2:sub2}
	\end{subfigure}
	\begin{subfigure}{0.155\textwidth} 
		\includegraphics[width=\linewidth]{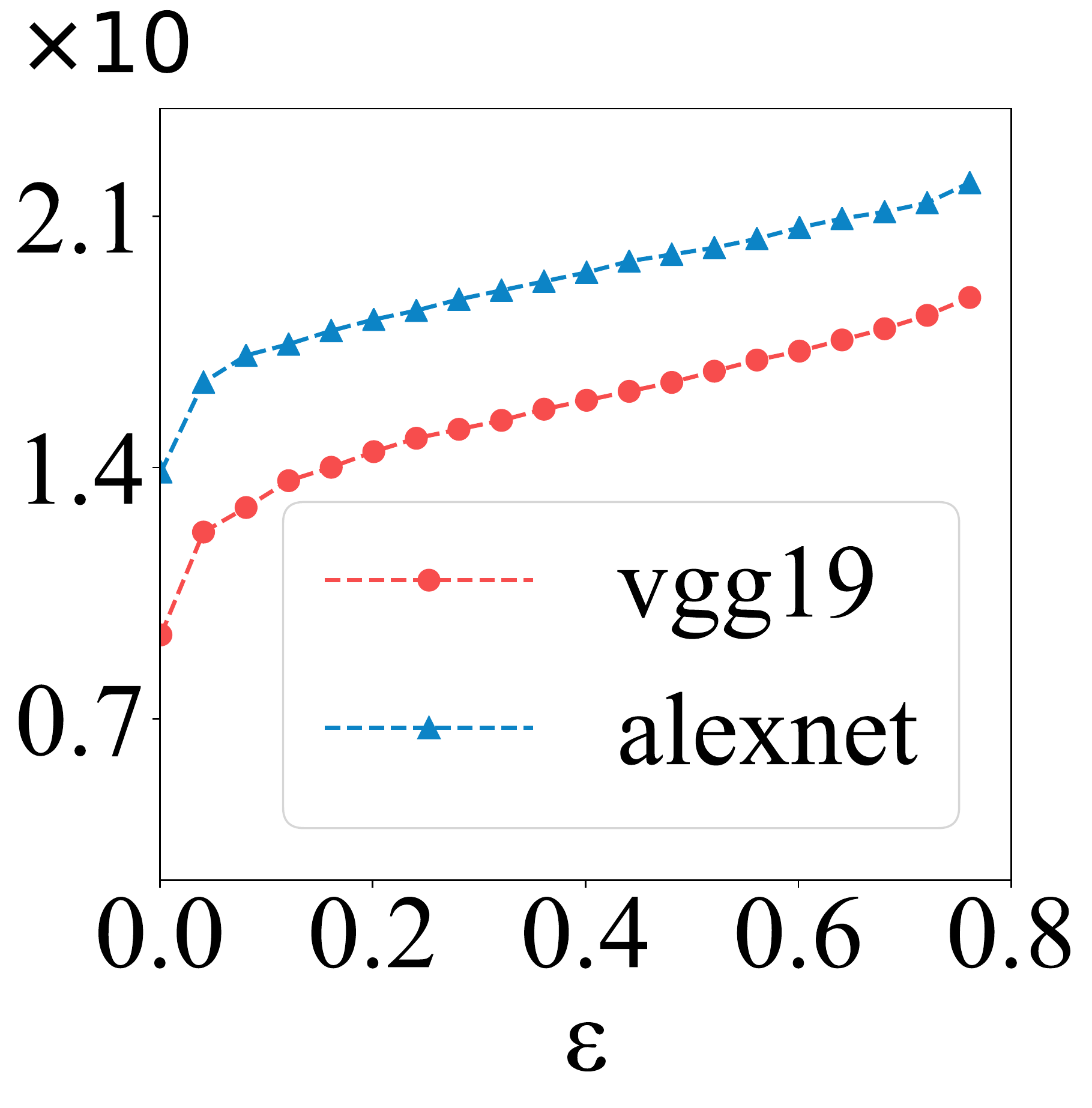}
		\caption{$\ell_{\infty}\ (p=\infty)$}
		\label{p3:sub3}
	\end{subfigure}	
	\caption{
		\label{p:total}
		(ImageNet) The $\varepsilon$-robustness radii ($y$-axis) of the same input point for VGG19 and AlexNet while the values of $\varepsilon$ vary from $0.001$ to $0.8$.
	}
\end{figure}

\subsection{$\varepsilon$-Robustness \& The Quality of NNs}
\textbf{Q2:} \emph{Can $\varepsilon$-robustness be used to reveal some potential problems about the quality of neural networks?}

In this experiment, we study whether $\varepsilon$-robustness analysis results can reveal the quality differences among the NNs with similar structure and reported accuracy. We trained 5 convolutional networks with almost the same structure on MNIST and they report almost the same accuracy on the test set.
\begin{itemize}
	\item Conv-Base achieves 98.97\% accuracy. It is a well-trained baseline model consisting of two convolutional layers, two fully connected layers, and 176,650 neurons in all.
	\item Conv-Overfit achieves 98.88\% accuracy. We overtrain the Cov-Base so that the training error is close to 0. It simulates the overfitting situation.
	\item Conv-NoPooling achieves 99.03\% accuracy. We remove the max-pooling layer which is designed to improve the robustness of neural networks.
	\item Conv-NoNorm achieves 98.41\% accuracy. We remove the normalization operation. It simulates a situation where a programmer forgot it.
	\item Conv-Imbalance report 98.66\% accuracy. We remove 90\% training and testing data in class `5'. It simulates the situation that a network is trained and tested on an imbalanced data set. 
\end{itemize}	
From the view of their reported accuracies on the test set, it is difficult to distinguish their quality differences. We study their $\varepsilon$-robustness performances further. We set $\varepsilon=0.02$. Figure \ref{fig:total} shows the results of $0.02$-robustness radii evaluated by EWRE. From it, we observe that: (1) The mean values of $0.02$-robustness radius on the Conv-Overfit, Conv-NoPooling, and Conv-NoNorm are smaller than that of the baseline model, Conv-Base. For the other values of $\varepsilon$, the relative relationship is consistent and this has been shown in the previous experiment. (2) The $0.02$-robustness radii decrease most by lack of the normalization operation (Conv-NoNorm). Usually, normalization is the first step for proceeding with various computer vision tasks. In some sense, our $\varepsilon$-robustness metric can provide evidence to support that it indeed plays an important role. 

Figure \ref{fig:Unbalance} presents the scatter plots of $0.02$-robustness radius $r$ for Conv-Imbalance and the images of the same class are rearranged together. For class `5', the $0.02$-robustness radii of points are distributed in the significantly lower value area and its mean value (marked by a triangle) is the smallest compared with other classes. It demonstrates that the unbalanced experimental situation can decrease the $\varepsilon$-robustness radius of the neural network for the unbalanced class. Our method has the potential to detect such problems.

%Most studies focus on the adversarial examples for software quality problems of neural networks. However, some points, which are classified correctly, have very small $\varepsilon$-robustness radii. We think they are worthy to be considered as well to improve the robustness of neural networks. Based on the application requirement, sometimes these points which have small $\varepsilon$-robustness radii can also be seen as ``bug'' points because they may be at high risk for a real-world scenario.

\begin{figure}
	\begin{subfigure}{0.23\textwidth}
		\includegraphics[width=1\linewidth]{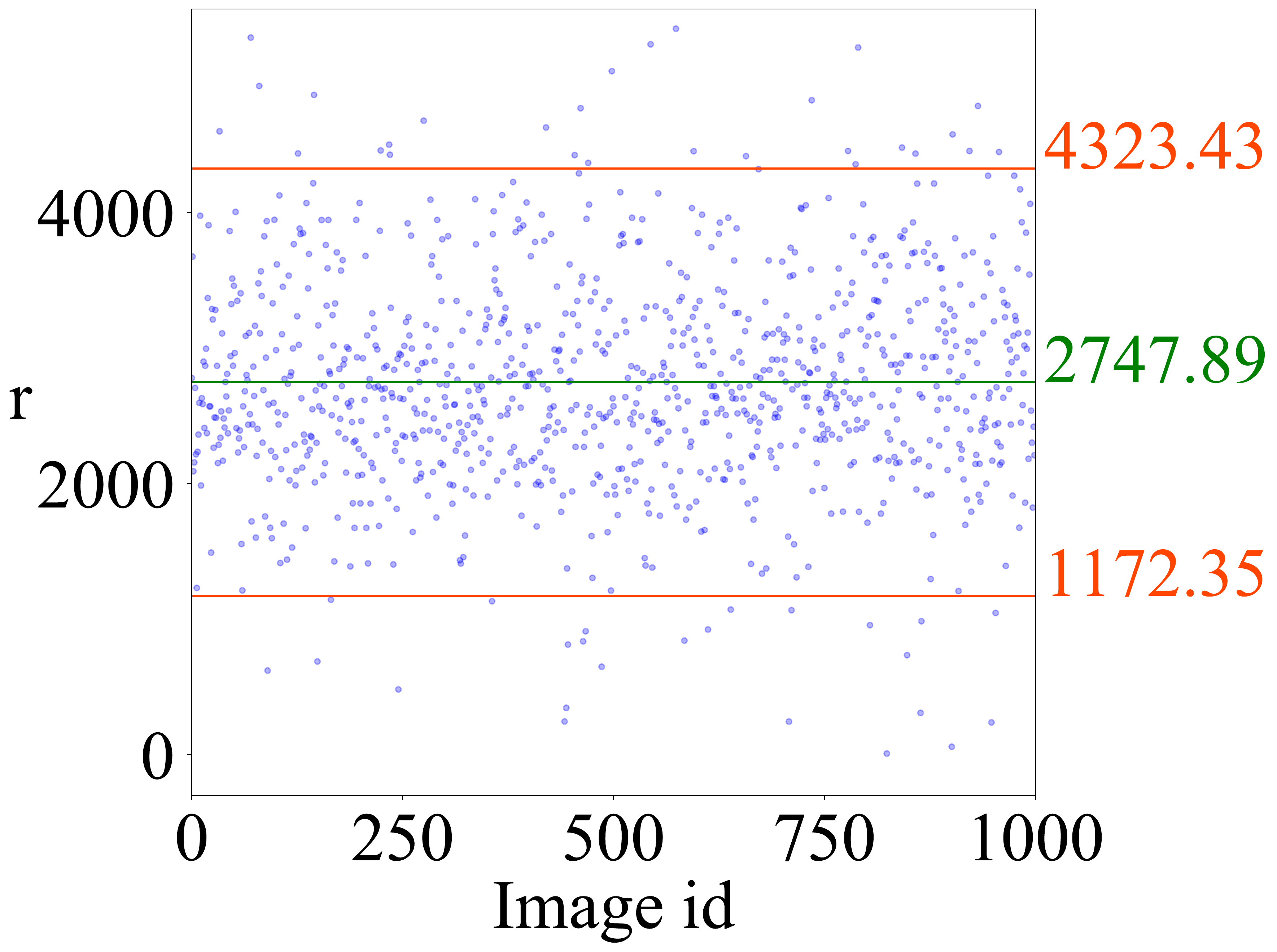}
		\caption{Conv-Base}
		\label{mnist:sub1}
	\end{subfigure}   %      \hfill  % 这个\hfill指令为插入弹性长度的空白，看情况选择加不加。
	\begin{subfigure}{0.23\textwidth}
		\includegraphics[width=\linewidth]{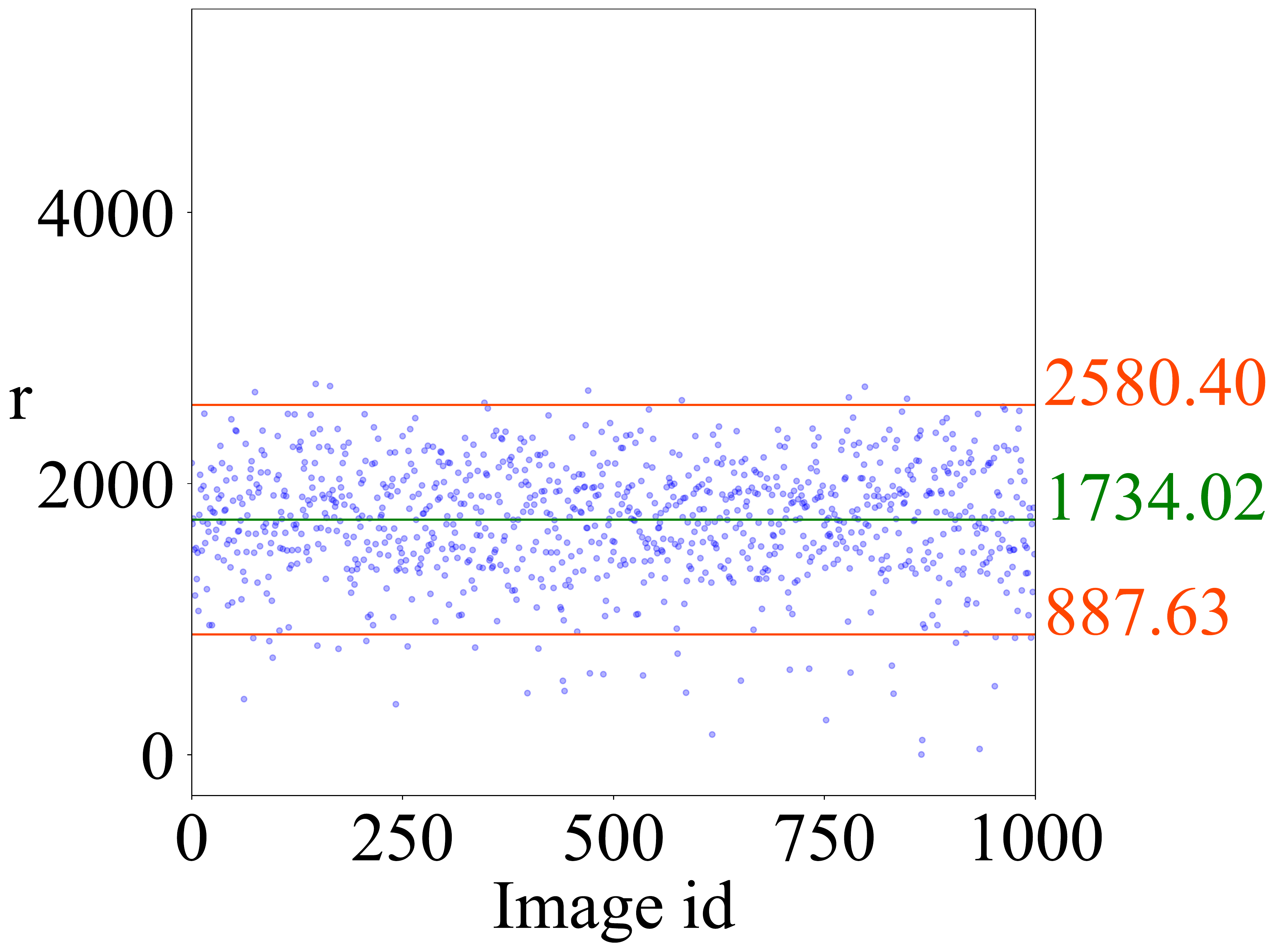}
		\caption{Conv-Overfit}
		\label{mnist:sub2}
	\end{subfigure}
	
	\begin{subfigure}{0.23\textwidth} 
		\includegraphics[width=\linewidth]{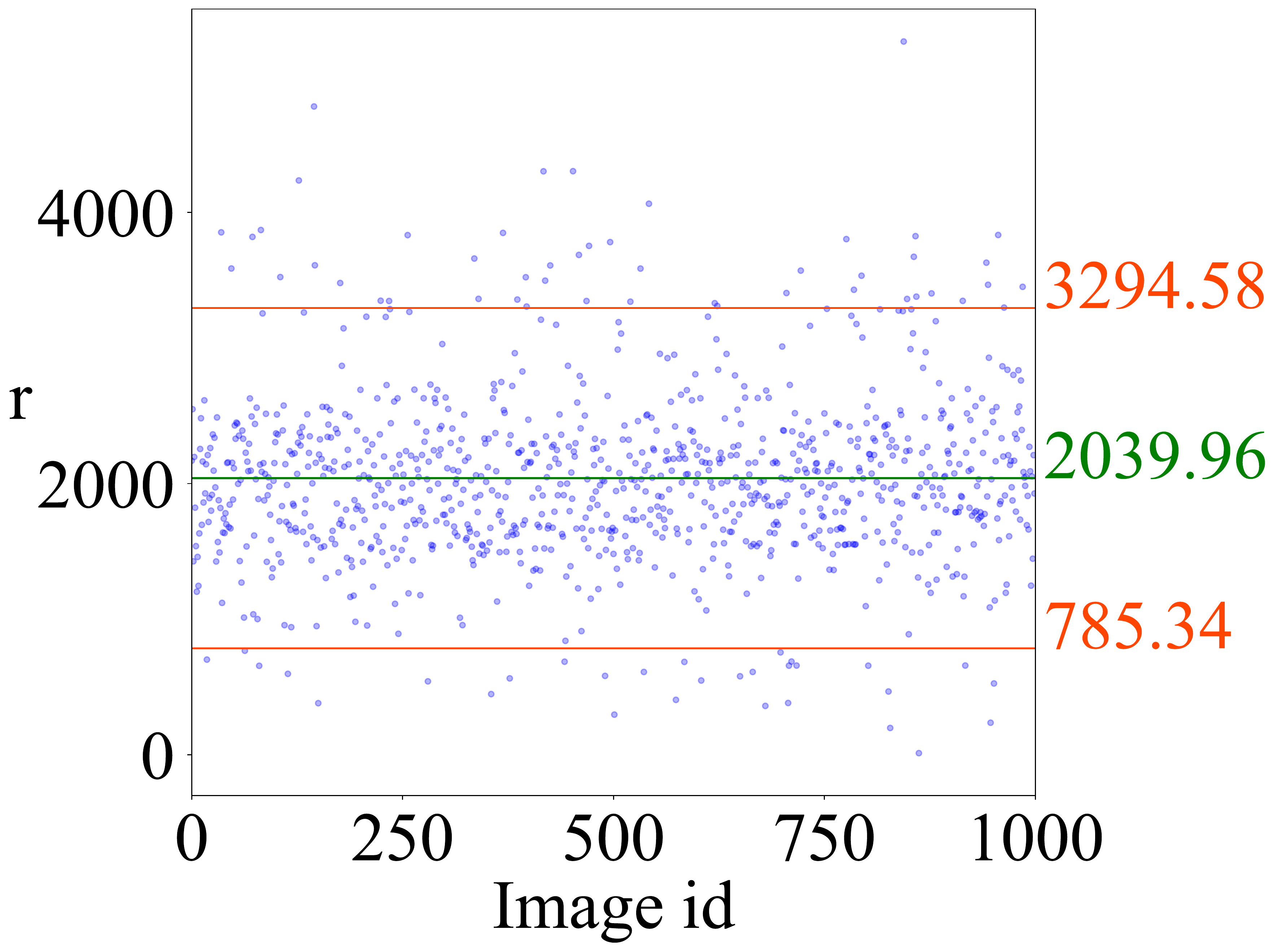}
		\caption{Conv-NoPooling}
		\label{mnist:sub3}
	\end{subfigure}
	\begin{subfigure}{0.23\textwidth} 
		\includegraphics[width=\linewidth]{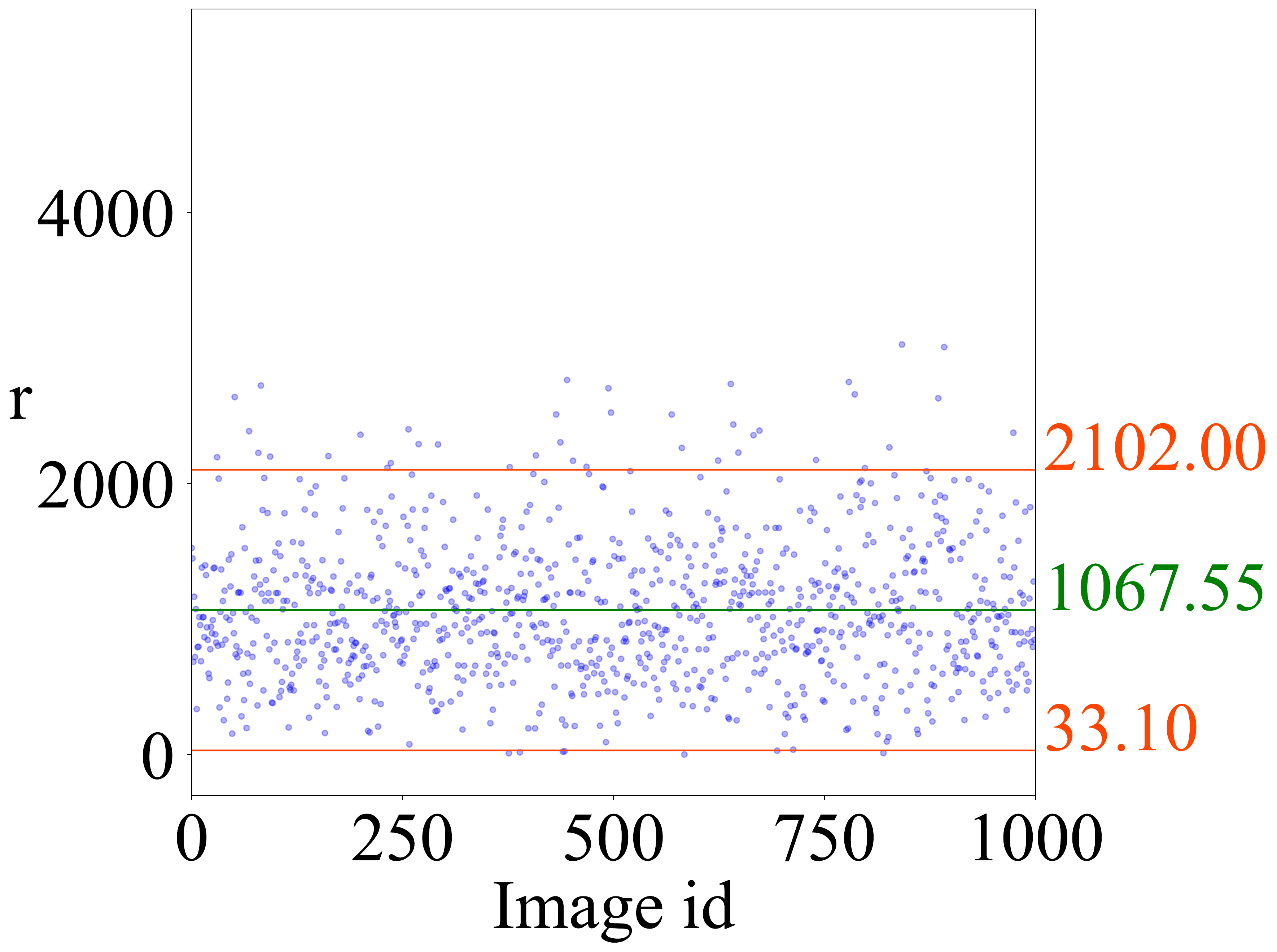}
		\caption{Conv-NoNorm}
		\label{mnist:sub4}
	\end{subfigure}
	\caption{
		$\varepsilon$-robustness radii distribution on 1000 (random 100 for each class) test images of MNIST for four neural networks ($p$=$2$). Green (middle) lines mark means. Upper (lower) red lines mark the mean plus (minus) twice the standard deviation. }
	\label{fig:total}
\end{figure}

\begin{figure}
	\begin{subfigure}{0.23\textwidth} 
		\includegraphics[width=1\linewidth]{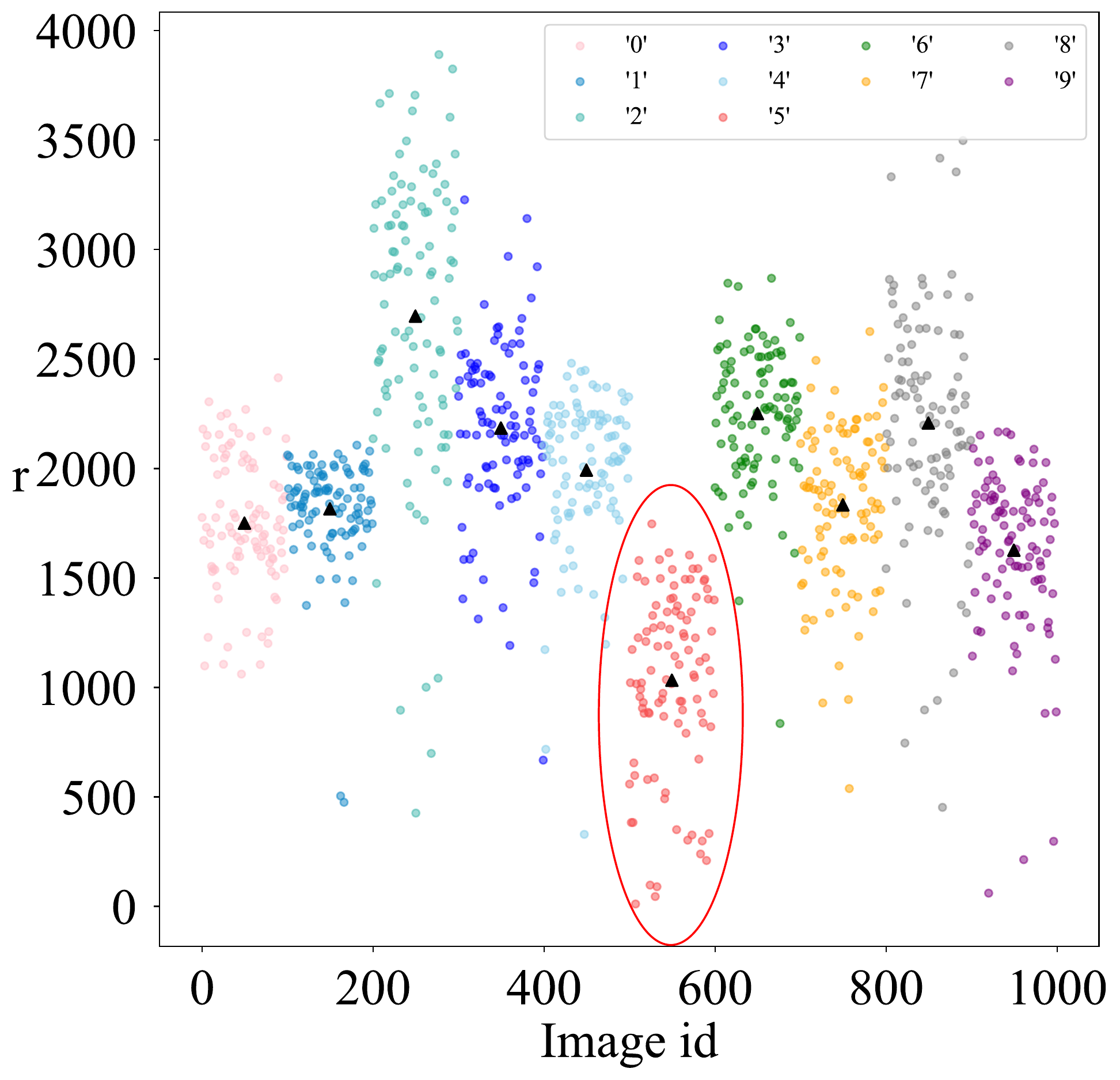}
		\caption{Conv-Imbalance (p=2)}
		\label{unbanlance1}
	\end{subfigure} 
	\begin{subfigure}{0.23\textwidth}
		\includegraphics[width=\linewidth]{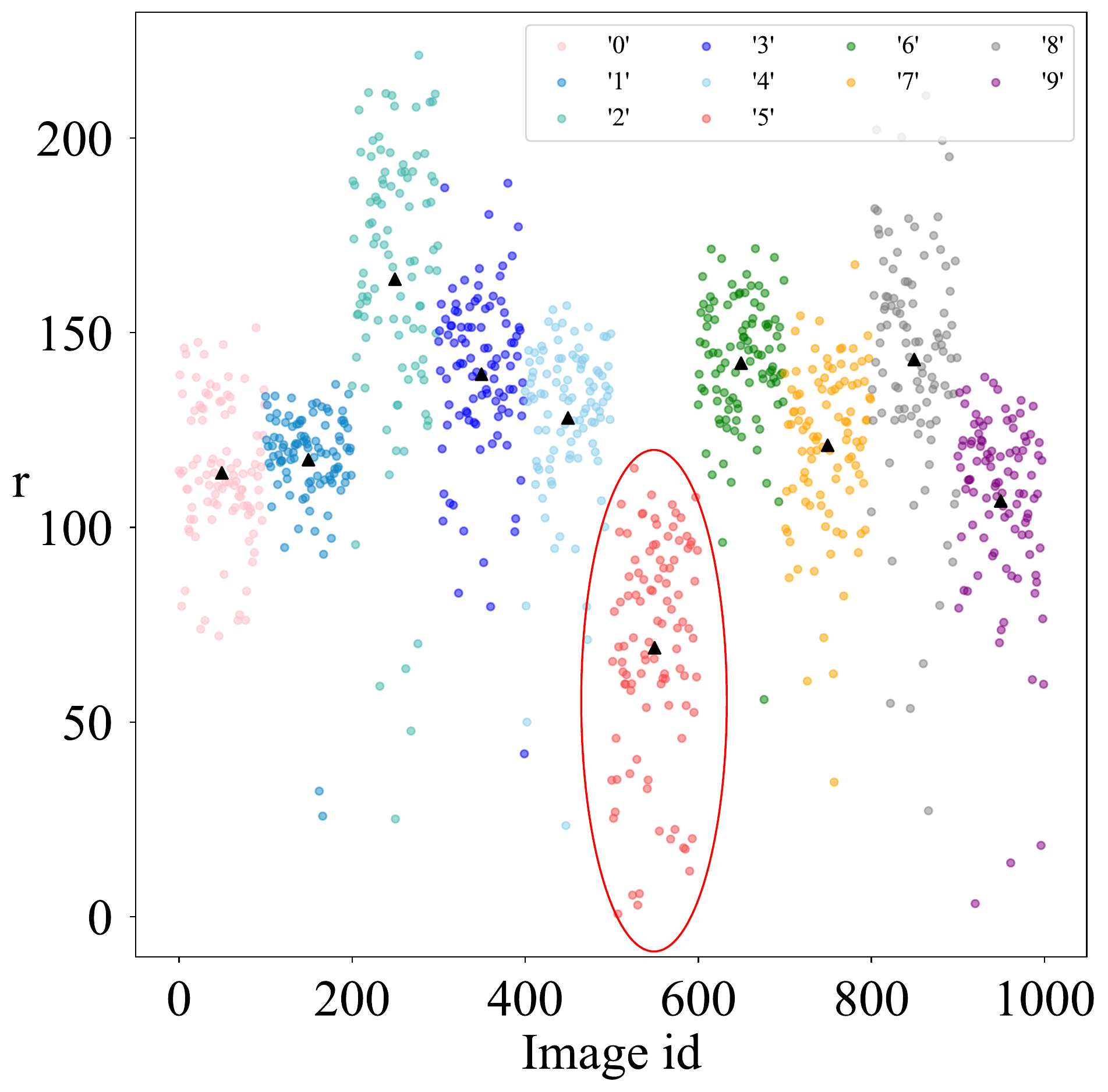}
		\caption{Conv-Imbalance (p=$\infty$)}
		\label{unbanlance2}
	\end{subfigure}
	\caption{
		$\varepsilon$-robustness radii on 1000 (random 100 for each class) test images of MNIST for the neural network trained on imbalanced data set. Each color represents a class. Points in class `5' are in red color and circled. Black triangles denote their means.}
	\label{fig:Unbalance}
\end{figure}

\subsection{$\varepsilon$-Robustness of Misclassified Points}
\textbf{Q3:} \emph{How about the $\varepsilon$-robustness of the regions around misclassified points or adversarial examples?}

Conventional robustness analysis can not be applied to the regions around misclassified points adversarial examples because these regions are not robust according to the definition. In practical application, an input is possibly in the regions around some misclassified points, which are also worthy to be analyzed. Thus, we designed the experiments to analyze the $\varepsilon$-robustness of four popular DNNs (ResNet18, RegNetX, DenseNet121 and DPN92) in the regions around misclassified points for CIFAR-10.

First, we evaluate the average $\varepsilon$-robustness radii of these four neural networks via EWRE for correctly-classified points and the results are shown in Table \ref{Tad}. 
Then, for each neural network, we run EWRD on all misclassified points in the test set and assign $r $ to the average $\varepsilon$-robustness radius. Figure \ref{fig:bar} shows the percent of $\varepsilon$-robust regions around misclassified points. We observe that the performance of neural networks in these regions is not necessarily very poor. Take RegNetX as an example,  4.18\% adversarial examples satisfy $0.01$-robustness with perturbation radius 52.86. That means although RegNetX commits errors at the centers of these regions, it often (with probability>99\%) works well in the neighborhoods. Besides, as we expected, there will be more regions satisfying $\varepsilon$-robustness if we can tolerate more adversarial examples. 

It is worth mentioning that EWRE can not be directly applied to the misclassified points, because it will invalidate the premise of binary search. Although the decremental search (gradually decreasing the radius) can be applied in theory, its practical efficiency is very low. Therefore, for such misclassified points, the radius should be set in advance according to application requirements.

\begin{table}
	\caption{(CIFAR-10) The average $\varepsilon$-robustness radii ($\ell_{\infty}$).}
	\begin{small}
		\begin{tabular}{lclrrr} % lcrC{0.40cm}rC{0.40cm}r
			\toprule
			\textbf{Networks}&
			$\varepsilon=0.01$ &$\varepsilon=0.1$ &$\varepsilon=0.15$ &$\varepsilon=0.20$\\
			\midrule
			ResNet18&59.47&68.53&71.24&73.49\\
			RegNetX&52.86&67.45&70.27&72.68\\
			DenseNet121&41.17&62.36&64.31&61.75\\
			DPN92&34.44&49.19&54.87&59.83\\
			\bottomrule
		\end{tabular}
	\end{small}
	\label{Tad}
\end{table}	

\begin{figure}
	\includegraphics[width=0.75\linewidth]{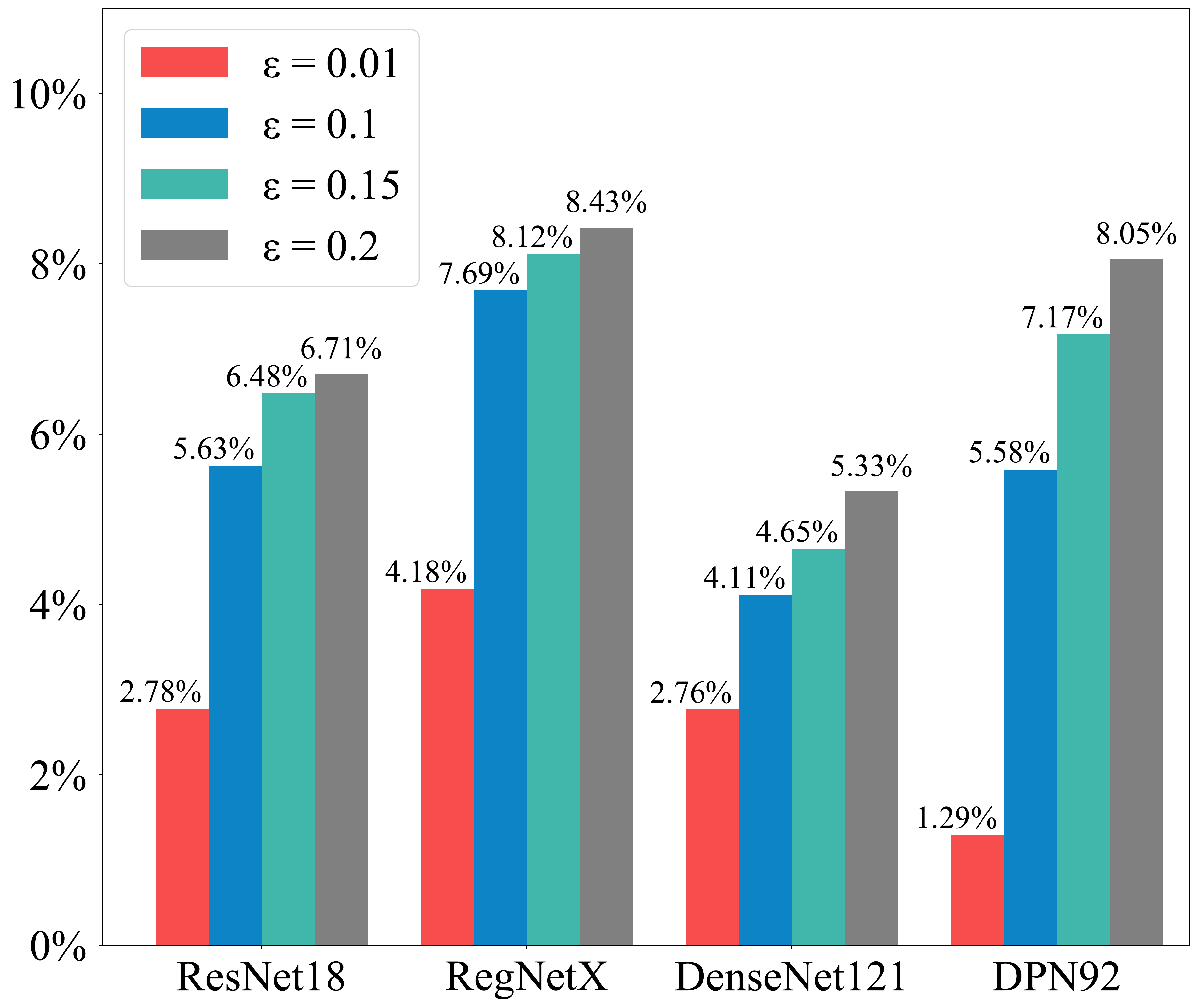}
	\caption{The percentage of regions reported to be $\varepsilon$-robustness around misclassified points.}
	\label{fig:bar}
\end{figure}

\subsection{Efficiency}
\textbf{Q4:} \emph{How about the efficiency of EWRD and EWRE?}

In this section, we test the efficiency of our algorithms for various networks and parameter settings. Firstly, we compared the efficiency of EWRD with ERAN on the same networks only for reference as they focus on different problems. When the value of $\varepsilon$ is very small, the $\varepsilon$-robustness decision problem and conventional robustness verification problem are very close. So we compare the running time via setting $\varepsilon \in \{0.0001,0.001\}$ in our experiment. Table \ref{T2} shows the results of the average running time. We observe that EWRD performs much more efficiently than ERAN. For conventional robustness verification problem, the soundness methods may encounter efficiency problems for large-scale networks and datasets. 

\begin{table}
	\centering
	\caption{The average running time of EWRD and ERAN on different neural networks. ($p$=$\infty$)}
	\resizebox{\linewidth}{!}{ %\textwidth
		\begin{tabular}{lclrrr} % lcrC{0.40cm}rC{0.40cm}r
			\toprule
			\multirow{2}{*}{\textbf{Networks}}&
			\multicolumn{2}{c}{Information} &\multicolumn{2}{c}{EWRD(s)} &\multirow{2}{*}{ERAN(s)}\\
			\cmidrule(lr){2-3}  \cmidrule(lr){4-5} 
			&\#Layers&\#Neurons&$\varepsilon=0.0001$&$\varepsilon=0.001$\\
			\midrule
			FullyCon-MNIST&4&3,072&4.80&0.38&26.90\\
			ConvBig-MNIST&6&48,064&2.68&0.55&3.03\\
			ConvSuper-MNIST&6&88,544&5.16&0.63&32.04\\
			FullyCon-CIFAR10&7&6,144&5.87&0.51&529.89\\
			ConvBig-CIFAR10&6&62,464&5.03&0.51&94.83\\
			ResNet18-CIFAR10&18&558,080&40.07&3.34&2971.37\\
			\bottomrule
		\end{tabular}
	}
	\label{T2}
\end{table}

Next, we test the average running time (seconds) of EWRD and EWRE for different $\varepsilon$ values on different networks and data sets which are shown in Table \ref{T1}. We can see that the dimensionality of inputs has the greatest impact on running time. Even for the Alexnet on ImageNet,  which contains much fewer layers compared with DenseNet121 on CIFAR-10, EWRD and EWRE take much more running time. With the increase of $\varepsilon$, the running time decreases first and then increases, which is consistent with the functional relationship between $N$ and $\varepsilon$ in Algorithm 1.

\begin{table*}
	\caption{The average running time of EWRD and EWRE on different neural networks. ($p$=$\infty$)}
	\resizebox{\textwidth}{!}{ %\textwidth
		\begin{tabular}{lcrrrrrrrrr} % lcrC{0.40cm}rC{0.40cm}r
			\toprule
			\multirow{2}{*}{\textbf{Networks}}&
			\multicolumn{2}{c}{Information} & \multicolumn{2}{c}{\textbf{$\varepsilon=0.001$}} &\multicolumn{2}{c}{\textbf{$\varepsilon=0.01$}} &\multicolumn{2}{c}{\textbf{$\varepsilon=0.10$}} &\multicolumn{2}{c}{\textbf{$\varepsilon=0.20$}}\\
			\cmidrule(lr){2-3} \cmidrule(lr){4-5}  \cmidrule(lr){6-7} \cmidrule(lr){8-9} \cmidrule(lr){10-11}
			&\#Layers&Data sets&EWRD(s) &EWRE(s)&EWRD(s)&EWRE(s)&EWRD(s)&EWRE(s)&EWRD(s)&EWRE(s)\\
			\midrule
			FullyCon&4&MNIST&0.31&3.56&0.02&0.35&0.41&4.97&1.38&16.19\\
			Conv1&4&MNIST&0.73&8.60&0.06&0.79&0.98&11.63&3.16&37.87\\
			ResNet18&18&CIFAR-10&3.34&40.85&0.30&3.40&4.05&50.12&12.89&153.40\\
			RegNetX&20&CIFAR-10&4.03&46.84&0.37&4.33&5.47&63.90&17.52&208.30\\
			DPN92&92&CIFAR-10&7.50&90.02&0.71&8.54&10.11&118.04&30.68&405.11\\
			DenseNet121&121&CIFAR-10&8.60&101.45&0.91&9.71&12.06&133.32&35.80&437.68\\
			Alexnet&8&ImageNet&34.2&393.34&3.38&38.86&44.73&525.13&165.50&1942.96\\
			VGG19&19&ImageNet&49.43&544.03&4.86&54.36&64.30&738.65&216.80&2706.32\\
			ResNet50 &50&ImageNet&46.31&462.19&4.42&45.35&63.30&602.64&227.88&2169.50\\
			\bottomrule
		\end{tabular}
	}
	\label{T1}
\end{table*}

We further study the effect of $\alpha$ and type $\beta$ on the execution efficiency of Algorithm 1. The smaller $\alpha$ and $\beta$ are, the more sampling points the algorithm needs. Figure \ref{AandBtime} shows that, with the exponential decrease of  $\alpha$ and $\beta$, the execution time of the algorithm does not increase exponentially which is attributed to the good property of the normal distribution. Although the sampling points $N \rightarrow \infty$ when $\alpha \rightarrow 0$ and $\beta \rightarrow 0$, this trend is not so fast. Besides, for different values of $p$, the execution efficiency of the sampling algorithm is also different as the time complexities are different as described in section \ref{Dsec}. For $p$=$1$, EWRE and EWRD consume the longest time compared with other norms.
%uniform sampling in $B_\infty(x_*,r)$ is the most efficient, and the average time complexity of sampling in $B_1(x_*,r)$ is $\mathcal{O}(n\log n)$. Figure \ref{AandBtime} also indicate that EWRD costs the shortest time for the case of p=$\infty$, and the longest time for the case of p=1.
\begin{figure}
	\begin{subfigure}{0.23\textwidth}
		\includegraphics[width=1\linewidth]{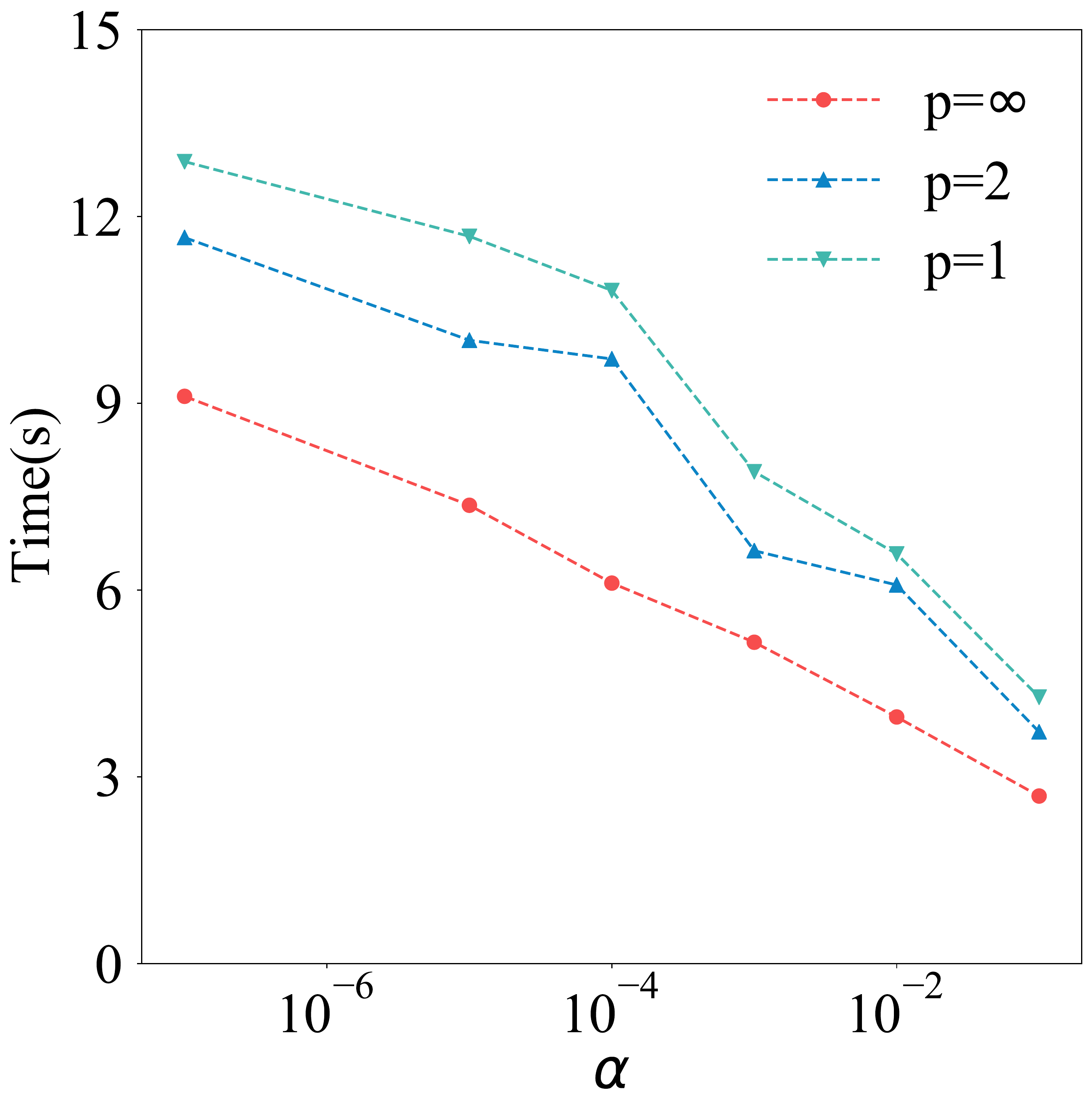}
		\label{alphatime}
	\end{subfigure}
	\begin{subfigure}{0.23\textwidth}
		\centering   
		\includegraphics[width=\linewidth]{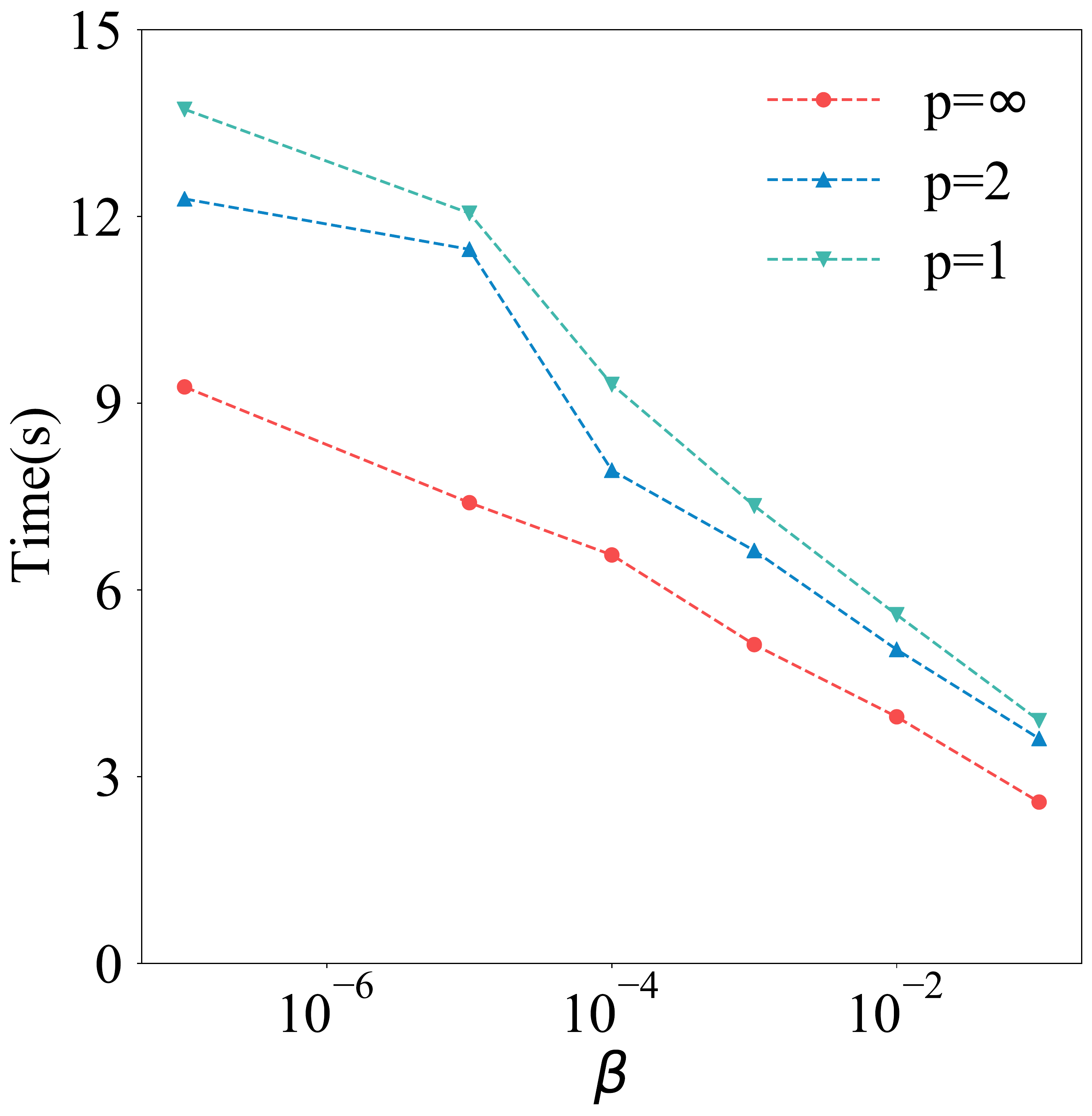}
		\label{betatime}
	\end{subfigure}
	\caption{
		\label{AandBtime}
		Comparing the running time of EWRD with respect to different values of $\alpha$ and $\beta$ on a given input. When $\alpha$ varies, $\beta=0.001$. When $\beta$ changes, $\alpha=0.001$. The $x$-axis is in logarithmic scale. Network: RegNetX. Data set: CIFAR-10. 
	}
\end{figure}

\subsection{Multi-labels Extension}
\textbf{Q5:} \emph{How about the $\varepsilon$-robustness in multi-labels extension?}

In practical application, a neural network can make a mistake but it does not mean a disaster. Figure \ref{fig:car} shows two examples in which ResNet18 recognizes a car as a truck and a plane respectively under random perturbations. When misclassifying a car as a truck, there may be no dangers for a self-driving system to make a decision to avoid the vehicle ahead. But it may be dangerous for the second case (misclassifying a car as a plane) as the self-driving system will make a wrong decision of no evasive action. Figure \ref{fig:omega} shows the values of the $\varepsilon$-robustness radius of a given input image (labeled `Car') on ResNet18 calculated by EWRE when $\Omega=\{`Car', `Truck'\}$. Compared with the case of $\Omega=\{`Car'\}$, the $\varepsilon$-robustness radius of the multi-label extension ($\Omega=\{`Car', `Truck'\}$) is significantly improved. When the $\varepsilon$-robustness radius is $42.21$, the risk of recognizing the car as a label out of $\Omega$ is less than $0.001$ under a random perturbation.
\begin{figure}
	\centering
	\includegraphics[width=0.9\linewidth]{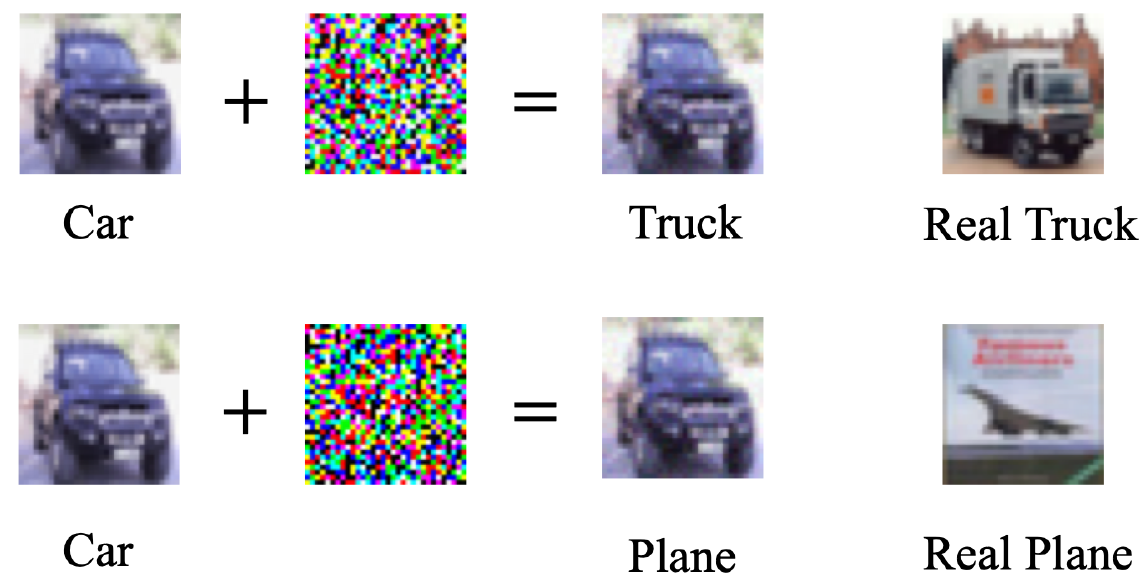}
	\caption{Adversarial examples on ResNet18 generated by adding random perturbation to images with label `Car'.}
	\label{fig:car}
\end{figure}

\begin{figure}
	\centering
	\includegraphics[width=0.8\linewidth]{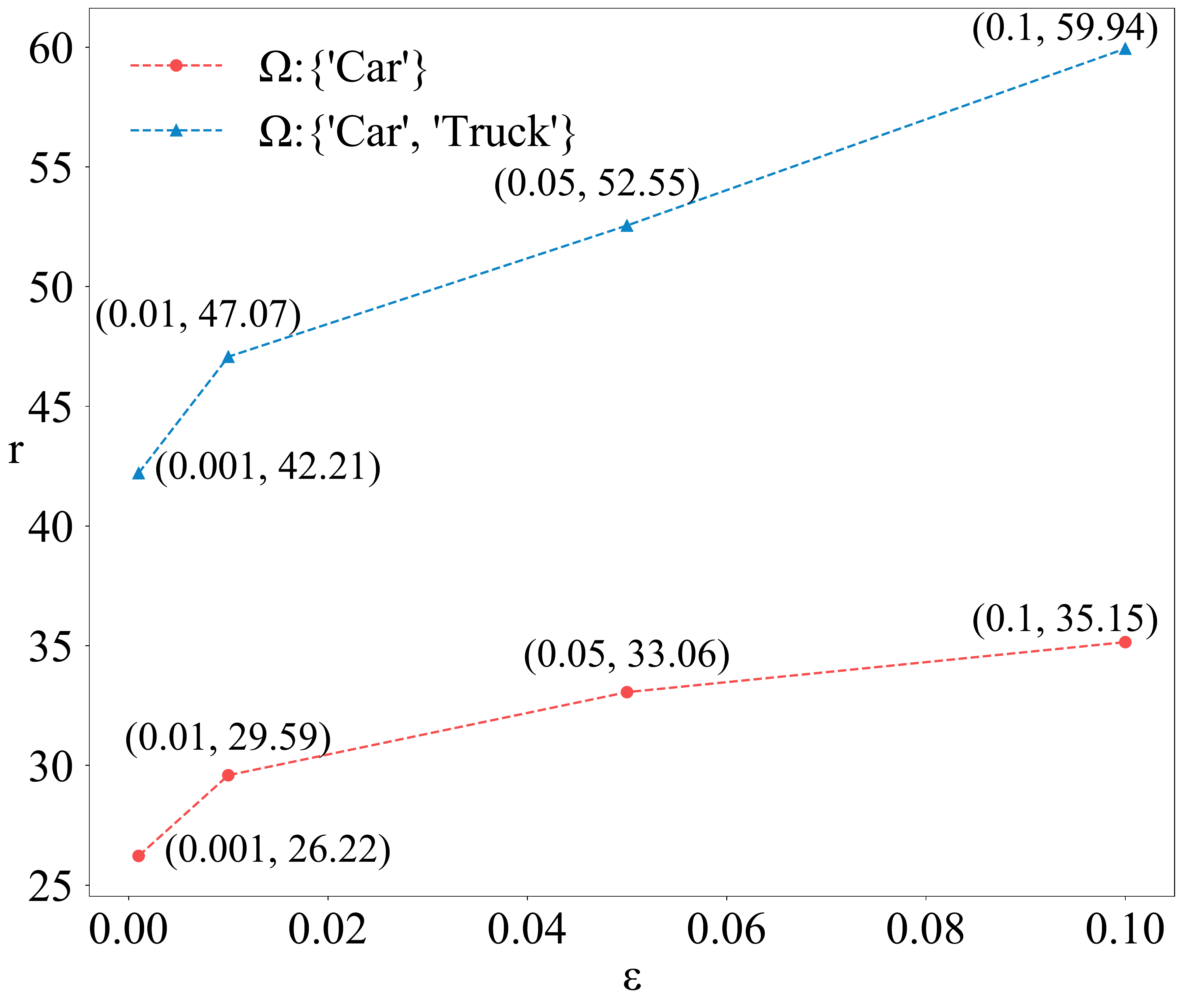}
	\caption{$\varepsilon$-robustness radii for a given image with the real label `Car' on ResNet18. ($p$=$\infty$)}
	\label{fig:omega}
\end{figure}

\section{Related Work}
In recent years, there have been a lot of works on certifying or evaluating the robustness of NNs.

Conventional robustness verification can be mainly divided into three categories. (1) Precise methods are based on Mixed Integer Linear Programming (MLP) \cite{FischettiJ18,DuttaJST18} or SMT \cite{KatzBDJK17,Ehlers17}. For specific types of NN, these methods have both completeness and soundness. It means that if the decision procedure returns ``robust'', the neural network is robust w.r.t. the specified perturbation radius, otherwise, there exist adversarial examples. However, these methods can only tackle small networks at present; (2) Incomplete methods which apply overapproximation techniques to scale up to medium-sized networks, which are sound but may fail to prove robustness even if it holds, such as abstract interpretation \cite{GehrMDTCV18,SinghGMPV18,SinghGPV19}, linear approximations\cite{WengZCSHDBD18,WongK18}, semidefinite relaxations \cite{RaghunathanSL18} or search space discretization \cite{HuangKWW17}. These methods suffer from a loss of precision when the neural networks are deep. Besides, they are limited to the types of networks and norms. (3) Estimation methods, including the method based on extreme value theory \cite{WengZCYSGHD18} and randomized smoothing \cite{CohenRK19,JiaCWG20}, can scale to large networks (e.g. ImageNet classifier). These methods have neither soundness nor completeness which means when the decision procedure returns ``robust'', the network can be either robust or have adversarial examples. Although randomized smoothing provides a probability guarantee, the radius which can be verified around a point vanishes when $p>2$ as recent studies \cite{Hayes20,Kumar0GF20} show.

In quantitative analysis or verification, some methods are proposed to estimate the proportion of inputs violating a property. So they also can be applied to estimate the proportion of adversarial examples. Webb et al. \cite{WebbRTK19} use a multi-level splitting sampling method to estimate the proportion of adversarial examples. However, it has no formal guarantees about the estimation results. In 2019, Baluta et al.\cite{BalutaSSMS19} proposed a white-box verification approach based on the model counting technique. It can be used to estimate the interval of adversarial examples' numbers with a formal guarantee. However, this approach is limited to binarized neural networks and takes 20 hours to verify robustness for a point when the network has more than 50, 000 parameters. These works focus on quantitative analysis (computing the number or proportion of adversarial examples) which is \#$P$-hard. Moreover, the estimated value can not be directly used to determine whether a system meets safety standards because the real value can be either over or below the threshold and the probability of making a wrong decision is unknown. After introducing $\varepsilon$, we reduce the computational complexity to $PP$. In this context, we can provide an efficient method to answer whether the proportion of adversarial examples is less than $\varepsilon$ with provable guarantees.

Our statistical approach to $\varepsilon$-robustness decision problem also can be regarded as statistical model checking \cite{ClarkeFLHJL08,ClarkeZ11}. Statistical model checking is traditionally used to reason about properties specified in a stochastic temporal logic for analyzing the behavior of complex stochastic systems. To our best knowledge, it has not been applied to analyze neural networks. Besides, there is no general method that can determine the sample size required for an arbitrary system \cite{AghaP18} and the time complexity of the decision procedure is related to the system and properties.
\section{Conclusion}
We formalize a robustness metric named $\varepsilon$-weakened robustness which can be used to analyze the reliability and stability of deep neural networks. Compared with the conventional robustness analysis, $\varepsilon$-robustness analysis can provide conclusive results about more regions of NNs. As the $\varepsilon$-robustness decision problem is $PP$-complete, we propose a probabilistic algorithm with user-controllable error bound to solve it. Based on the decision oracle, we further provide an efficient algorithm to find the maximum $\varepsilon$-robustness radius. The time complexity of our algorithm is polynomial in the dimension and size of the network. Experimental results present the scalability of our methods to large neural networks.

\section{Notification}
This work was completed in November 2020, but the paper was not published. The patent was applied in 2020.
\bibliography{bibfile}

\section{Appendices}

\subsection{Complexity Proof}

First, we will introduce a PP-complete problem called MAJORITY SAT (MAJSAT($\geq2^{n-1}$)): given a CNF-formula $\varphi$ with $n$ variables $\{x_1...x_n\}$, is it satisfied by at least half of the possible truth assignments? To facilitate follow-up proof, we denote the variant version as MAJSAT'($>2^{n-1}$) which can be easily reduced from MAJSAT in polynomial time. Let 
\begin{equation*}
	\begin{aligned}
		\psi=&x_1\vee (\neg x_1 \wedge x_2 \wedge ...\wedge x_n)\\
		&(x_1\vee x_2) \wedge ...\wedge (x_1\vee x_n))
	\end{aligned}
\end{equation*}

We know that $\psi$ has $2^{n-1}+1$ truth assignments.

\begin{equation*}
	\varphi'= (\neg p \wedge \varphi)\vee (p \wedge \psi)
\end{equation*}
where $p$ is an added variable and $\varphi'$ can be easily transformed to CNF in polynomial time via distributive law.  Assuming $\varphi$ has $m$ truth assignments. So $\varphi'$ has $m+2^{n-1}+1$ truth assignments and $n+1$ variables.  We can query MAJSAT' whether $\varphi'$ is satisfied by at least $(2^{n+1}/2=2^n)$ assignments. If yes, we have $m+2^{n-1}+1>2^n$ which can be rewritten as $m \geq 2^{n-1}$. Otherwise, we have $m<2^{n-1}$.

\begin{theorem}
	If the input space $\mathbb{D}^n$ is a countable set, $\varepsilon$-robustness decision problem is $PP$-complete.
\end{theorem}

\begin{proof}
	We first show that the problem is in PP. We prove that the $\varepsilon$-robustness decision problem is solvable by a probabilistic algorithm in polynomial time, with an error probability of less than $1/2$ for all instances. Consider a probabilistic algorithm that, chooses an $x$ uniformly at random in $B_p(x_*, r)$. Then, the algorithm tosses a coin that every outcome is equal.
	
	\begin{itemize}
		\item If the coin comes up heads, chooses an $x$ uniformly at random in $B_p(x_*, r)$. Then, 	check if $x$ makes $T(F,x,\Omega)=1$. If yes, output SAT. If no, output UNSAT.
		
		\item If the coin comes up tails, the algorithm will outputs SAT with probability $(\varepsilon-1/ |\mathbb{D}|^{n+1})$ and UNSAT with probability $(1-\varepsilon+1/ |\mathbb{D}|^{n+1})$.
	\end{itemize}
	
	If the neural network is $\varepsilon$-robust, $p_{r}>1-\varepsilon$ and it is equivalent to $p_{r}\geq 1-\varepsilon+1/ |\mathbb{D}|^{n}$. The algorithm will always output SAT with probability:
	\begin{equation}
		\begin{aligned} 
			&\frac{1}{2}(1-\varepsilon+\frac{1}{|\mathbb{D}|^{n}})+\frac{1}{2}(\varepsilon-\frac{1}{|\mathbb{D}|^{n+1}})\\
			&=\frac{1}{2}+\frac{1}{2}(\frac{1}{|\mathbb{D}|^{n}}-\frac{1}{|\mathbb{D}|^{n+1}})>\frac{1}{2}
		\end{aligned} 
	\end{equation}
	
	If the neural network is not $\varepsilon$-robust, $p_{r}\leq1-\varepsilon$. The algorithm will always output SAT with probability:
	\begin{equation}
		\begin{aligned} 
			&\frac{1}{2}(1-\varepsilon)+\frac{1}{2}(\varepsilon-\frac{1}{|\mathbb{D}|^{n+1}})\\
			&=\frac{1}{2}-\frac{1}{2}\frac{1}{|\mathbb{D}|^{n+1}}<\frac{1}{2}
		\end{aligned} 
	\end{equation}

	Thus this algorithm puts $\varepsilon$-weakened robustness decision problem in PP.
	
	Next, we show the reduction from the MAJSAT'($>2^{n-1}$).  Any SAT formula $\varphi=C_1\wedge C_2\wedge...\wedge C_m$ can be transformed into a DNN with ReLUs. Input variables are mapped to each $t_i$ node according to the definition of clause $C_i$ and $y_i$ will be equal to 1 if clause $C_i$ is satisfied, and will be 0 otherwise. Fig \ref{figproof1} shows the method to construct a ``Not'' gadget and  a clause $C_m$. The clause gadget can be regarded as calculating the expression:
	\begin{equation*}
		y_m=1-\max(0,1-\sum_{j=1}^{k}l_j)
	\end{equation*}
	
	Then, we add an extra output node $o_2$ with input 1 and set edge weight to $m-0.5$ in the last layer (Fig \ref{figproof2} is a schematic). If $o_1>o_2$, the network outputs lable $l_1$. Otherwise, it outputs lable $l_2$. Therefore, the output will be $l_1$ if and only if all clauses are simultaneously satisfied.
	
	Let $\Omega=\{l_1\}$, $p=\infty$, $r=1$,$\varepsilon=1/2$ and $x_*=[0,0, ..., 0]$. Then we have that the SAT formula $\varphi$ has at least $2^{n-1}$ assignments if and only if the neural network is $1/2$-weakened robust.
\end{proof}

\begin{figure}[h!]
	\centering
	\includegraphics[width=0.6\linewidth]{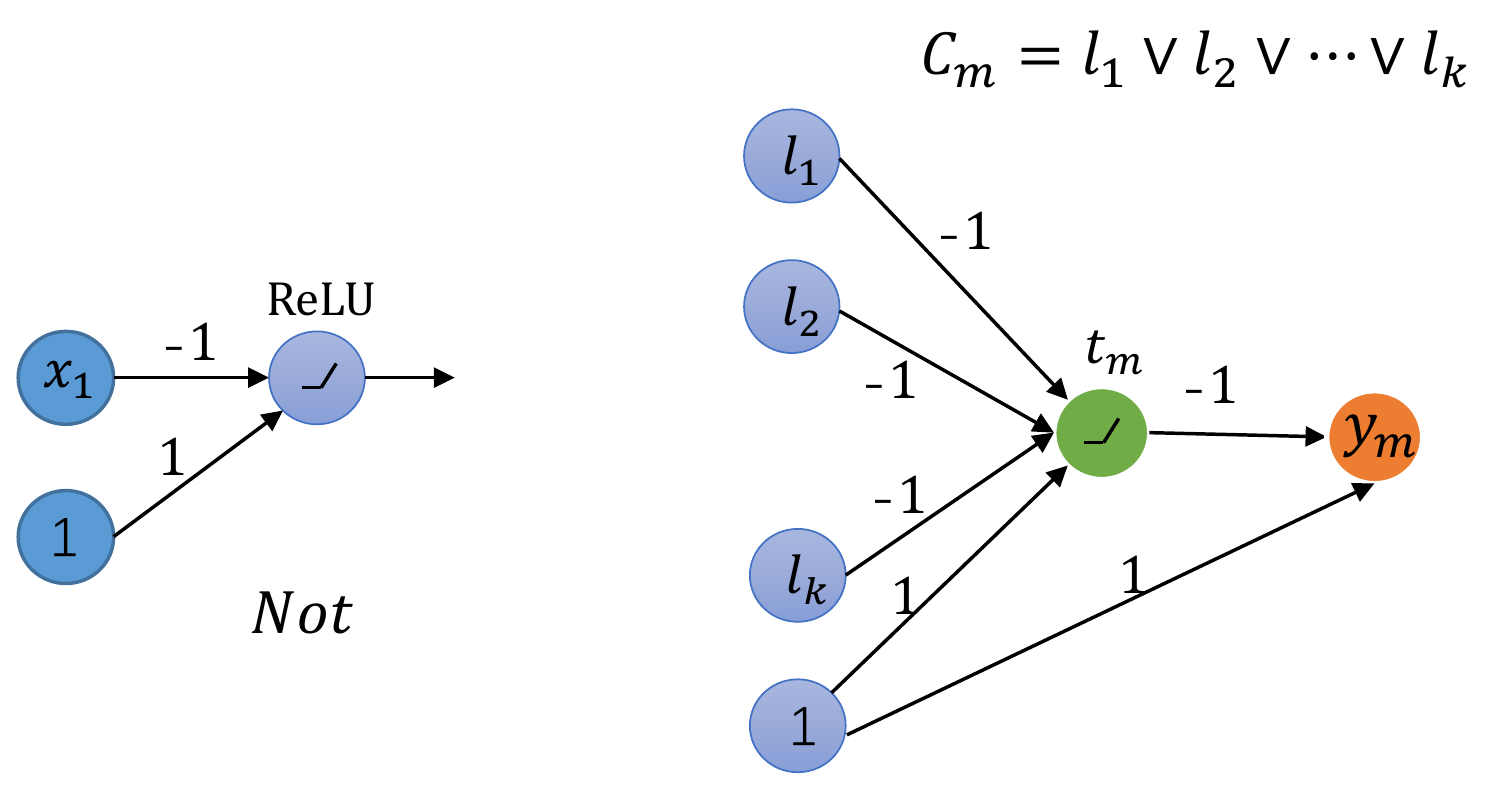}
	\caption{The ``Not'' gadget and ``Clause'' gadget.}
	\label{figproof1}
\end{figure}

\begin{figure}[h!]
	\centering
	\includegraphics[width=0.6\linewidth]{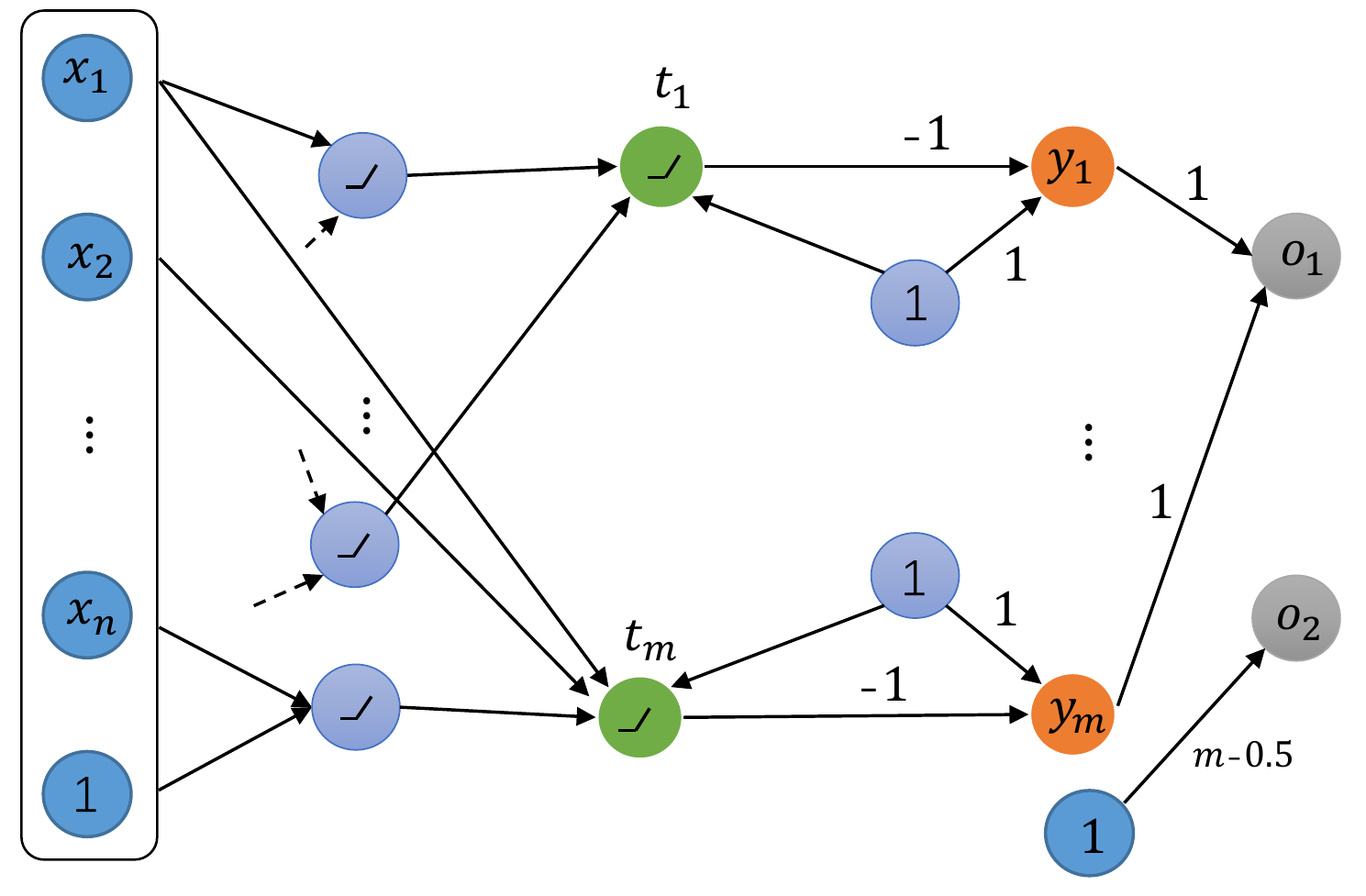}
	\caption{The network}
	\label{figproof2}
\end{figure}

\subsection{Sampling Algorithms}\label{sample_Algorithm}
In this section, uniform sampling algorithms in $\ell_1$, $\ell_2$ and $\ell_{\infty}$-norm balls are given.

Algorithm 1 outlines the method to sample a point uniformly in $B_1(x_*,r)$.

\begin{algorithm}[h]
	\caption{Sample a point uniformly in $B_1(x_*,r)$}
	\KwIn{$x_*$- test point, $r$-distortion radius}
	\KwOut{ $y=(y^{(1)}...y^{(n)})$}
	$y' \leftarrow (y'^{(0)}...y'^{(n)})$\;
	$y'^{(0)} \leftarrow 0$\;
	\For{$i \leftarrow 1$ \KwTo $n$}{
		$y'^{(i)} \sim \mathcal{U}(0,r)$\ \tcp*[f]{Uniform distribution}
	}
	$y'^{(1)}...y'^{(n)} \leftarrow Sort(y'^{(1)}...y'^{(n)})$;\tcp*[f]{Ascending order}
	
	\For{$i \leftarrow 1$ \KwTo $n$}
	{
		$y^{(i)} \leftarrow RandomSign \cdot (y'^{(i)}-y'^{(i-1)})$\;
	}
	$y\leftarrow y+x_*$\;
	\Return $y$
\end{algorithm}

Algorithm 1 makes use of the uniform spacings to construct the random sampling point~\cite{Devroye86}. We know that $\sum_{i=1}^{n}|y'^{(i)}-y'^{(i-1)}|=y'^{(n)}\leq r$, so the sample point is in $B_1(x_*,r)$. The average-case time complexity of Algorithm 1 is $\mathcal{O}(n\log n)$ and the worst case is $\mathcal{O}(n^2)$.

Algorithm 2 outlines the method to sample a point uniformly in $B_2(x_*,r)$.
\begin{algorithm}[h]
	\caption{Sample a point uniformly in $B_2(x_*,r)$}
	\KwIn{$x_*$- test point, $r$-distortion radius}
	\KwOut{ $y=(y^{(1)}...y^{(n)})$}
	$y' \leftarrow (y'^{(1)}...y'^{(n)})$\;
	\For{$i \leftarrow 1$ \KwTo $n$}{
		$y'^{(i)} \sim \mathcal{N}(0,1)$\;
	}
	$s \leftarrow \sum\nolimits_{i=1}^n  (y'^{(i)})^2$;  \tcp*[f]{$\left\|y'\right\|_2^2$}
	
	$\kappa \leftarrow \Gamma(n/2, s/2)^{\frac{1}{n}}$; \tcp*[f]{Incomplete gamma function}
	
	\For{$i \leftarrow 1$ \KwTo $n$}
	{
		$y^{(i)} \leftarrow  \frac{r \cdot \kappa \cdot y'^{(i)}}{\sqrt{s}}$\;
	}
	$y\leftarrow y+x_*$\;
	\Return $y$
\end{algorithm}

This algorithm was first proposed by Muller \cite{Muller}. It makes use of a strong property that the exponent part of joint Gaussian distribution is the same as the expression for calculating the $\ell_2$-norm. The time complexity of this method is $\mathcal{O}(n)$.

Algorithm 3 outlines the method to sample a point uniformly in $B_\infty(x_*,r)$ and the time complexity is $\mathcal{O}(n)$.
\begin{algorithm}[h]
	\caption{Sample a point uniformly in $B_\infty(x_*,r)$}
	\KwIn{$x_*$- test point, $r$-distortion radius}
	\KwOut{ $y=(y^{(1)}...y^{(n)})$}
	\For{$i \leftarrow 1$ \KwTo $n$}{
		$y^{(i)} \sim \mathcal{U}(-r,r)$\;
	}
	$y\leftarrow y+x_*$\;
	\Return $y$
\end{algorithm}

\end{document}